\documentclass{article}

\usepackage{microtype}
\usepackage{graphicx}
\usepackage{subfigure}
\usepackage{booktabs} 

\usepackage{hyperref}

\usepackage[accepted]{icml2024}

\usepackage{amsmath}
\usepackage{amssymb}
\usepackage{mathtools}
\usepackage{amsthm}
\usepackage{xspace}
\usepackage{enumitem}
\usepackage{booktabs,multirow,adjustbox,diagbox,threeparttable,tabularray}
\usepackage{amsmath,amsthm,amssymb,amsbsy,amsfonts,dsfont,pifont,bm,bbm,mathrsfs,mathtools,nicefrac}
\usepackage{listings,colortbl}

\newcounter{unnumberedsection}
\makeatletter
\newcommand{\unnumberedsection}[1]{%
    \section*{#1}%
    \addcontentsline{toc}{section}{#1}%
    \refstepcounter{unnumberedsection}%
    \def\cref@currentlabel{[unnumberedsection][\arabic{unnumberedsection}][]#1}%
    \def\@currentlabelname{#1}%
}

\usepackage[capitalize,noabbrev]{cleveref}
\crefname{section}{Sec.}{Secs.}
\Crefname{section}{Section}{Sections}
\crefname{table}{Tab.}{Tabs.}
\Crefname{table}{Table}{Tables}
\crefname{figure}{Fig.}{Figs.}
\Crefname{figure}{Figure}{Figures}
\crefname{equation}{Eq.}{Eqs.}
\Crefname{equation}{Equation}{Equations}
\crefname{theorem}{Thm.}{Thms.}
\Crefname{theorem}{Theorem}{Theorems}
\crefname{proposition}{Prop}{Props.}
\Crefname{proposition}{Proposition}{Propositions}
\crefname{lemma}{Lem.}{Lems.}
\Crefname{lemma}{Lemma}{Lemmas}
\crefname{remark}{Rem.}{Rems.}
\Crefname{remark}{Remark}{Remarks}
\crefname{corollary}{Cor.}{Cors.}
\Crefname{corollary}{Corollary}{Corollaries}
\crefname{algorithm}{Alg.}{Algs.}
\Crefname{algorithm}{Algorithm}{Algorithms}
\crefname{unnumberedsection}{}{}
\Crefname{unnumberedsection}{}{}
\theoremstyle{plain}
\newtheorem{theorem}{Theorem}[section]
\newtheorem{proposition}[theorem]{Proposition}

\theoremstyle{definition}

\theoremstyle{remark}
\newtheorem{remark}[theorem]{Remark}

\newcommand{\method}{\texttt{SMaRt}\xspace}

\newcommand{\tocite}[1]{{\color{red} [TOCITE]}}

\usepackage[textsize=tiny]{todonotes}

\icmltitlerunning{SMaRt: Improving GANs with Score Matching Regularity}

\begin{document}

\twocolumn[
\icmltitle{SMaRt: Improving GANs with Score Matching Regularity}



\icmlsetsymbol{equal}{*}

\begin{icmlauthorlist}
\icmlauthor{Mengfei Xia}{thu,bnrist}
\icmlauthor{Yujun Shen}{ant}
\icmlauthor{Ceyuan Yang}{ailab}
\icmlauthor{Ran Yi}{sjtu} 
\icmlauthor{Wenping Wang}{texam}
\icmlauthor{Yong-Jin Liu$^\dag$}{thu}
\end{icmlauthorlist}

\icmlaffiliation{thu}{Tsinghua University}
\icmlaffiliation{bnrist}{BNRist}
\icmlaffiliation{ant}{Ant Group}
\icmlaffiliation{ailab}{Shanghai AI Laboratory}
\icmlaffiliation{sjtu}{Shanghai Jiao Tong University}
\icmlaffiliation{texam}{Texas A\&M University}

\icmlcorrespondingauthor{Yong-Jin Liu}{liuyongjin@tsinghua.edu.cn}

\icmlkeywords{Machine Learning, ICML}

\vskip 0.3in
]



\printAffiliationsAndNotice{}  

\begin{abstract}

Generative adversarial networks (GANs) usually struggle in learning from highly diverse data, whose underlying manifold is complex.
In this work, we revisit the mathematical foundations of GANs, and theoretically reveal that the native adversarial loss for GAN training is insufficient to fix the problem of \textit{subsets with positive Lebesgue measure of the generated data manifold lying out of the real data manifold}.
Instead, we find that score matching serves as a promising solution to this issue thanks to its capability of persistently pushing the generated data points towards the real data manifold.
We thereby propose to improve the optimization of GANs with score matching regularity (\method).
Regarding the empirical evidences, we first design a toy example to show that training GANs by the aid of a ground-truth score function can help reproduce the real data distribution more accurately, and then confirm that our approach can consistently boost the synthesis performance of various state-of-the-art GANs on real-world datasets with pre-trained diffusion models acting as the approximate score function.
For instance, when training Aurora on the ImageNet $64\times64$ dataset, we manage to improve FID from 8.87 to 7.11, on par with the performance of one-step consistency model.
Code is available at \href{https://github.com/thuxmf/SMaRt}{https://github.com/thuxmf/SMaRt}.

\end{abstract}

\section{Introduction}\label{sec:intro}

During the last period, deep generative models have made significant improvements in a variety of domains, such as data generation~\cite{Karras2019AnalyzingAI,Karras2021AliasFreeGA,ho2020denoising,song2020score,dhariwal2021diffusion,Karras2022edm} and image editing~\cite{Shen2019InterpretingTL,Shen2020ClosedFormFO,zhu2022linkgan,meng2022sdedit,couairon2023diffedit}.
It is well recognized that, recent generative models, such as DALL$\cdot$E 2~\cite{Ramesh2022HierarchicalTI}, Stable Diffusion~\cite{Rombach2021HighResolutionIS}, GigaGAN~\cite{Kang2023ScalingUG}, and Aurora~\cite{zhu2023aurora}, have achieved unprecedented capability improvement of high-resolution image generation, among which, diffusion probabilistic models (DPMs) are the most prominent.
DPMs leverage the diffusion and denoising processes.
Their intrinsic intricate knowledge of data distribution and strong capability to scale up, make DPMs the most successful and potential options for generative modeling.
The other paradigm now dominant, generative adversarial networks (GANs)~\cite{goodfellow2014gan,Brock2018LargeSG}, introduce an implicit modeling.
Despite enabling expeditious generation, GANs are usually criticized for unsatisfactory visual quality and limited diversity when compared with DPMs, making GANs seem to be falling from grace on image generation tasks.
However, GAN remains a worthy tool considering its good performance on single-domain datasets (\textit{e.g.}, human faces)~\cite{Karras2019AnalyzingAI,Karras2021AliasFreeGA} and its interpretable latent space~\cite{Shen2019InterpretingTL,zhu2022linkgan}.

\definecolor{myred}{RGB}{148,44,59}
\definecolor{myblue}{RGB}{128,127,222}

\begin{figure}[t]
\centering
\includegraphics[width=0.8\linewidth]{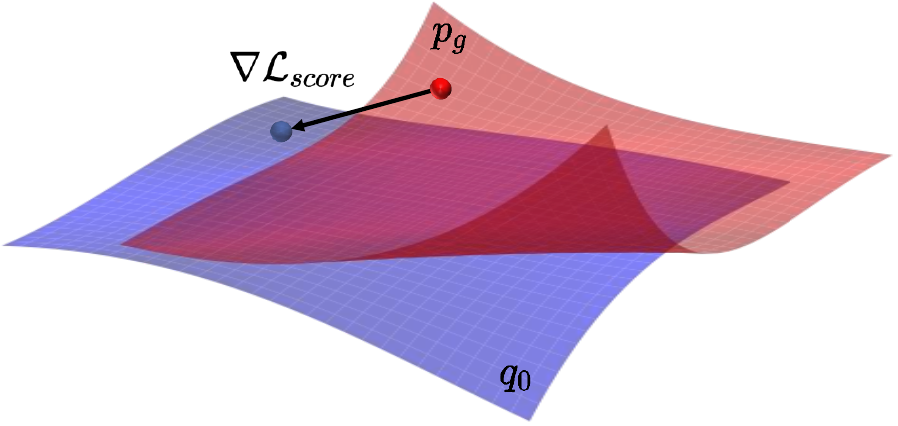}
\vspace{-5pt}
\caption{
    \textbf{Motivation scheme} of \method.
    \textbf{\textcolor{myred}{Red}} and \textbf{\textcolor{myblue}{blue}} surfaces denote the generated and real data manifolds, respectively.
    The positive-Lebesgue-measure subset of out-of-manifold generated samples leads to non-optimal constant generator loss, annihilating the gradient for generator.
    However, the proposed score matching regularity ($\mathcal L_{score}$ in \cref{eq:diffusion_loss}) provides complementary guidance, urging such a subset to move towards the real data manifold.
    In this case, generator loss regains to exert effective guidance aiding the generator distribution to converge to the real distribution.
}
\label{fig:motivation}
\vspace{-8pt}
\end{figure}

In this work, we dig into the mathematical foundations of GANs and reveal the necessary and sufficient conditions of optimality of generator loss.
We argue in \Cref{thm:nongrad} that positive-Lebesgue-measure difference sets of generated data manifold over real data manifold lead to constant but non-optimal generator loss, which annihilates the gradient and cancels effective guidance.
However, such non-optimality largely harms the synthesis performance, demonstrated in \Cref{thm:mode_collapse}.
More seriously, the gradient vanishing occurs frequently in practice.
Note that, real and generated data can be referred to as low-dimensional manifolds embedded in the high-dimensional pixel space, leading to the probability of transversal intersection or non-intersection equaling to 1~\cite{arjovsky2017principled}.
This indicates that the difference set of generated data over real data manifold almost always has positive Lebesgue measure.

Based on the above analyses, we are devoted to designing an effective methodology to tackle this obstacle.
We propose a universal solution, \textbf{S}core \textbf{Ma}tching \textbf{R}egulari\textbf{t}y, namely \method, leveraging score matching to facilitate GAN training.
The theoretical foundation is that, score matching pushes out-of-data-manifold generated samples towards the data manifold throughout, summarized in \Cref{thm:score_matching} and demonstrated in \cref{fig:motivation}.
Revealing this rigorous mathematical foundation, \method persistently provides gradient for generator, enforcing the generator distribution to support only on the data manifold.
Afterwards, the generator loss could regain the correct and effective guidance towards data distribution.
Our motivation is intuitive -- GAN loss focuses only on the \textit{generated and real data manifold}, while the score matching on the \textit{whole space} manages to serve as a regularity to facilitate GAN training.
By doing so, we succeed on alleviating the gradient vanishing issue.
Hence, our work offers a new perspective on improving GAN performance.
Given the rapid improvement in seminal works, the editing on a well-studied latent space~\cite{Shen2019InterpretingTL,Shen2020ClosedFormFO,zhu2022linkgan}, and the strong compatibility with the involvement of 3D-aware image synthesis~\cite{Chan2022,piGAN2021,gao2022get3d,gu2022stylenerf,shi2023pof3d,shi2022improving}, we believe that our work could encourage more studies in the field of visual content generation.

\vspace{9pt}

\section{Related work}\label{sec:related}

\noindent\textbf{GANs and improved GAN training.}
GANs~\cite{goodfellow2014gan} have become one of the main paradigms of generative models for high-quality image generation.
Thanks to the rapidly and significantly improvement on the sampling quality~\cite{Karras2017ProgressiveGO,Karras2018ASG, Karras2019AnalyzingAI,Karras2021AliasFreeGA,Kang2023ScalingUG,zhu2023aurora}, GANs are introduced to various downstream applications, including text-to-image synthesis~\cite{Reed2016GenerativeAT,Kang2023ScalingUG,zhu2023aurora}, and image-to-image translation~\cite{Isola2016ImagetoImageTW,Rai2018UnpairedIT,Huang2018MultimodalUI,Lee2018DRITDI,Park2019SemanticIS,Park2020ContrastiveLF}.
In particular, style-based GANs~\cite{Karras2018ASG,Karras2019AnalyzingAI} have shown impressive ability on single-domain datasets (\textit{e.g.}, human faces) and interpretable latent space~\cite{Shen2019InterpretingTL,zhu2022linkgan}.
%
However, GANs severely suffer from the famous ``gradient vanishing''~\cite{arjovsky2017principled} dilemma, restricting further development of synthesis quality and diversity.
To this end, WGAN~\cite{arjovsky2017wgan} replaces the native KL-divergence with Wasserstein distance as the GAN loss, inspired by optimal transportation.
Besides, progressive training has been widely studied in GAN literature~\cite{piGAN2021,Karras2017ProgressiveGO,Karras2018ASG}, thanks to its efficacy in improving training stability and efficiency.
Theoretically, \method can be considered as a regularity compatible with existing GAN training strategies, effectively addressing the GAN training obstacles.

\noindent\textbf{DPMs and efficient DPM sampling}.
DPMs~\cite{sohl2015deep,ho2020denoising,song2020score} introduce a novel scheme of generative model, trained by optimizing the variational lower bound.
Benefiting from this breakthrough, DPMs achieve high generation fidelity, and even beat GANs on image generation.
Therefore, various works followed with promising results, including video synthesis~\cite{ho2022video}, conditional generation~\cite{choi2021ilvr,lhhuang2023composer}, and text-to-image synthesis~\cite{Ramesh2022HierarchicalTI,Rombach2021HighResolutionIS,Saharia2022PhotorealisticTD}.
However, DPM employs an iterative refinement via thousands of denoising steps, suffers from a slow inference speed.
Efficient DPM sampling explores shorter denoising trajectories rather than the complete reverse process, while ensuring the synthesis performance.
One representative category introduces knowledge distillation~\citep{SalimansH22,luhman2021knowledge,song2023consistency,luo2023diffinstruct}.
Despite respectable performance with one step~\citep{song2023consistency,luo2023diffinstruct}, they require expensive distillation stages, leading to poor applicability.

\section{Method}\label{sec:method}

\subsection{Background on GANs and DPMs}\label{subsec:3.1}

Denote by $\mathbf x$ the training data with an unknown distribution $q_0(\mathbf x)$.
GANs involve a generator $G$ and a discriminator $D$, to map random noise $\mathbf z$ to sample and discriminate real or generated samples, respectively~\cite{goodfellow2014gan}.
Formally, GANs endeavor to achieve Nash equilibrium via the following two losses:
\vspace{-2pt}
\begin{align}
\mathcal L_G&=-\mathbb E_{\mathbf z}[\log D(G(\mathbf z))], \label{eq:g_loss} \\
\mathcal L_D&=-\mathbb E_{\mathbf x}[\log D(\mathbf x)]-\mathbb E_{\mathbf z}[\log(1-D(G(\mathbf z)))], \label{eq:d_loss}
\end{align}
where $\mathbf z$ is random noise embedded in the latent space.

\begin{figure*}[t]
\centering
\includegraphics[width=1.0\textwidth]{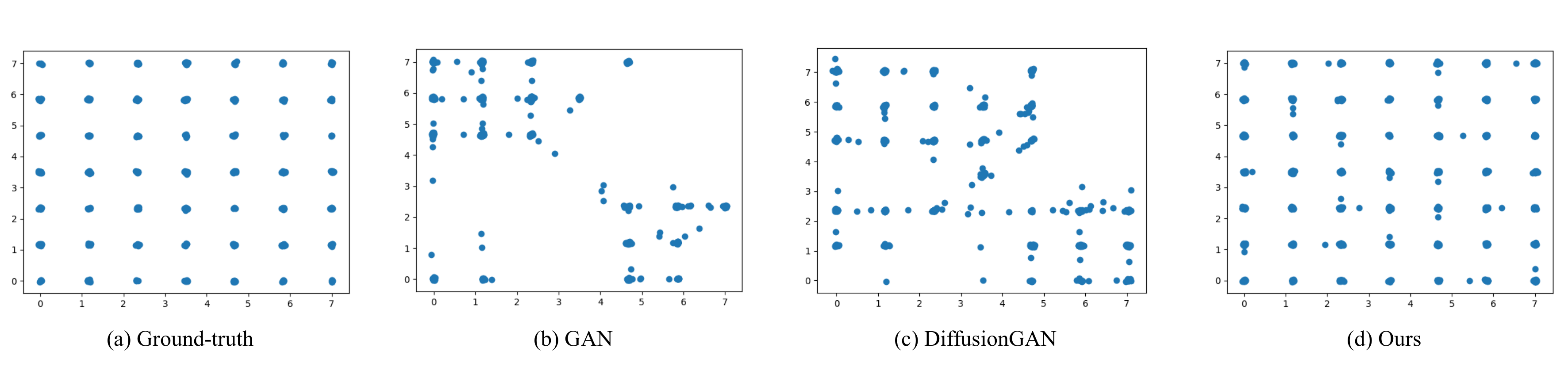}
\vspace{-20pt}
\caption{
    \textbf{Visualization} of discrete distribution example.
    The demonstrated toy data is simulated by a mixture of 49 2-dimensional Gaussian distributions with extremely low variance.
    Each data sample is a 2-dimensional feature tensor.
    Following~\citet{wang2022diffusiongan}, we train a small GAN model, whose generator and discriminator are both parameterized by MLPs, with two 128-unit hidden layers and Leaky ReLU activation functions.
    We show (a) the true data samples, (b) the generated samples from vanilla GAN, (c) the generated samples from DiffusionGAN~\cite{wang2022diffusiongan}, and (d) the generated samples from our \method.
    As is demonstrated, vanilla GAN and DiffusionGAN \textit{fail} to address all samples onto the data manifold discretely, \textit{i.e.}, the generated samples tend to be continuous and \textit{out of the grids}.
    As a comparison, our \method can successfully synthesize discrete samples, whose distribution coincides with the ground-truth.
}
\label{fig:toy_model}
\end{figure*}

On the other hand, DPMs~\cite{sohl2015deep,song2020score,ho2020denoising} define a forward diffusion process $\{\mathbf x_t\}_{t\in[0,T]}, T>0$ by gradually corrupting the initial information of $\mathbf x$ with Gaussian noise, such that for any timestep $0<t\leqslant T$, we have the transition distribution:
\begin{align}\label{eq:tran}
q_{0|t}(\mathbf x_t|\mathbf x)=\mathcal N(\mathbf x_t;\alpha_t\mathbf x,\sigma_t^2\mathbf I),
\end{align}
where $\alpha_t,\sigma_t>0$ are differentiable functions of $t$.
The selection of $\alpha_t,\sigma_t$ is referred to as the \textit{noise schedule}.
Denote by $q_t(\mathbf x_t)$ the marginal distribution of $\mathbf x_t$, DPM fits $\mathcal N(\mathbf x_T;\mathbf 0,\sigma^2\mathbf I)$ with $q_T(\mathbf x_T)$ for some $\sigma>0$, and the signal-to-noise-ratio (SNR) $\alpha_t^2/\sigma_t^2$ is strictly decreasing w.r.t. $t$~\cite{kingma2021variational}.
DPMs utilize the noise prediction model $\boldsymbol \epsilon_\theta(\mathbf x_t,t)$, to approximate the score function from $\mathbf x_t$, where the optimal parameter $\theta^*$ can be optimized by the objective below through denoising score matching:
\begin{align}\label{eq:dpm_loss}
\theta^*=\mathop{\arg\min}_{\theta}\mathbb E_{\mathbf x,\boldsymbol\epsilon,t}\left[\|\boldsymbol \epsilon_\theta(\mathbf x_t,t) - \boldsymbol\epsilon\|_2^2\right],
\end{align}
where $\boldsymbol\epsilon\sim\mathcal N(\mathbf 0,\mathbf I)$, $\mathbf x_t=\alpha_t\mathbf x+\sigma_t\boldsymbol\epsilon$, and $t\sim\mathcal U[0,T]$.

\subsection{Revisiting GAN Training}\label{subsec:3.2}

We first delve into the theory of GANs, trying to analyze the dilemma encountered by GANs with deep findings.
Recall that when GANs achieve the Nash equilibrium, we have the two equalities about $q_0$ and generator distribution $p_g$:
\begin{align}
D(\mathbf x)=\frac{q_0(\mathbf x)}{q_0(\mathbf x)+p_g(\mathbf x)},\quad p_g(\mathbf x)=q_0(\mathbf x),
\end{align}
and $\mathcal L_G$ in~\cref{eq:g_loss} reaches the minimum $\log2$.
However, we have the following theorem. Proof is in \cref{subsec:a.1}.

\begin{theorem}\label{thm:nongrad}
Let $A,B$ be sets with positive $d$-dimensional Lebesgue measure, \textit{i.e.}, $\mu_d(A)>0,\mu_d(B)>0$.
Denote by $q_A(\mathbf x),q_B(\mathbf x)$ two distributions supported on $A,B$, respectively, \textit{i.e.}, $\mathrm{supp}\;q_A=\{\mathbf x\mid q_A(\mathbf x)\neq0\}=A$, $\mathrm{supp}\;q_B=B$.
Let $X\backslash Y=\{\mathbf x\mid\mathbf x\in X\text{ and }\mathbf x\notin Y\}$.
When $D$ reaches the optimality, and if $\mu_d(A\backslash B)>0$, then
\begin{align}
-\int q_A(\mathbf x)\log\frac{q_B(\mathbf x)}{q_A(\mathbf x)+q_B(\mathbf x)}\mathrm d\mathbf x=+\infty.
\end{align}
\end{theorem}

Let $q_0=q_B,p_g=q_A$, \Cref{thm:nongrad} claims that $\mathcal L_G$ remains non-optimal constant and provides no gradient to the generator when the generated data has positive-measure difference set over data manifold.
Empirically, real data is embedded in a very low-dimensional manifold in the pixel space, and so is the generated data due to the low-dimensional latent space.
The two manifolds will almost always have zero-measure intersection (transversal intersection or non-intersection), and thus positive-measure difference set, since $A\backslash B=A\backslash(A\cap B)$.
Therefore, \Cref{thm:nongrad} is almost always the case during GAN training.

We then turn to the necessary and sufficient conditions of optimality of generator loss, which is summarized below.
The proof is addressed in \cref{subsec:a.2}

\begin{theorem}\label{thm:mode_collapse}
Following the settings in \Cref{thm:nongrad}, when $D$ reaches the optimality, the following inequality reaches its optimality if and only if $\mu_d(A\backslash B)=\mu_d(B\backslash A)=0$, and $\mu_d(\{\mathbf x\mid q_A|_{A\cap B}(\mathbf x)\neq q_B|_{A\cap B}(\mathbf x)\})=0$.
\begin{align}
-\int q_A(\mathbf x)\log\frac{q_B(\mathbf x)}{q_A(\mathbf x)+q_B(\mathbf x)}\mathrm d\mathbf x&\geqslant\log2,
\end{align}
where $f|_X(\mathbf x)$ equals to $f(\mathbf x)$ on $X$ and 0 out of $X$.
\end{theorem}

\Cref{thm:mode_collapse} claims that the optimality of generator loss is equivalent to \textit{the generator distribution coinciding with the real distribution almost everywhere}.
This gives an insight of the generator behavior, \textit{i.e.}, generator correctly imitates data distribution only if the generator loss achieves optimality.
Combining with \Cref{thm:nongrad}, once $(\mathrm{supp}\;p_g)\backslash(\mathrm{supp}\;q_0)$ has positive Lebesgue measure, the generator distribution has not coincided with the ground-truth yet but will no longer be updated.
To be more detailed, the generator loss fails due to the \textit{low dimension of the data and generator manifolds}.
The two manifolds will almost always have zero-measure intersection (transversal intersection or non-intersection), and thus positive-measure difference set.

We further give a toy example designed on discrete data distribution.
The toy data is simulated by a mixture of 49 2-dimensional Gaussian distributions with extremely low variance.
Each data sample is a 2-dimensional feature tensor.
To demonstrate the poor performance on discrete data distribution, following~\citet{wang2022diffusiongan}, we train a small GAN, whose generator and discriminator are both parameterized by MLPs, with two 128-unit hidden layers and Leaky ReLU activation functions.
As shown in \cref{fig:toy_model}, vanilla GAN and DiffusionGAN cannot handle the discrete data, synthesizing continuous samples.
In other words, vanilla GAN tends to synthesizing a positive-measure set of samples out of the data manifold (\textit{i.e.}, the 49 grids).
This directly leads to gradient vanishing due to \Cref{thm:nongrad}.

\subsection{Score Matching Regularity}\label{subsec:3.3}

Unlike GANs, DPMs focus on the \textit{whole pixel space} via score matching, due to the forward diffusion process, which diffuses the data distribution to the normal distribution.
Thanks to this, score matching manages to serve as a regularity to facilitate GAN training.

We first delve into the theory of score matching.
Recall that in the DDIM sampling process~\cite{song2020denoising}, one first calculates $\hat{\boldsymbol\epsilon_t}=\boldsymbol\epsilon_\theta(\mathbf x_t,t)$ for the intermediate noisy result $\mathbf x_t$.
With this $\hat{\boldsymbol\epsilon_t}$, one can predict an approximation $\hat{\mathbf x}_0$ of the clean data.
However, this predicted $\hat{\mathbf x}_0$ is usually of poor quality, and needs further refinement by the iterative diffusion and denoising process.
Formally, given a sample $\mathbf x$, the one-step refinement $\mathcal R$ process with noise $\boldsymbol\epsilon$ and timestep $t$ is defined as the following form:
\begin{align}\label{eq:refine}
\mathcal R(\mathbf x,\boldsymbol\epsilon,t):=\mathbf x+\frac{\sigma_t}{\alpha_t}(\boldsymbol\epsilon-\boldsymbol\epsilon_\theta(\alpha_t\mathbf x+\sigma_t\boldsymbol\epsilon,t)).
\end{align}

\vspace{-5pt}

Note that infinitely many one-step refinements with infinitesimal $t$ could pull any out-of-data-manifold point back to data manifold~\cite{Welling2011BayesianLV}.
We summarize this property below.
Proof is addressed in \cref{subsec:a.3}.

\begin{theorem}\label{thm:score_matching}
Denote by $\mathrm{dist}(\mathbf x)$ the distance between $\mathbf x$ and $\mathrm{supp}\;q_0$.
For any $\mathbf y\notin\mathrm{supp}\;q_0$, define a sequence of random variables $\mathbf y_0=\mathbf y$, $\mathbf y_{k+1}=\mathcal R(\mathbf y_k,\boldsymbol\epsilon_k,t)$ with $\boldsymbol\epsilon_k\sim\mathcal N(\mathbf 0,\mathbf I)$.
Then $\{\mathbf y_k\}_{k=0}^{\infty}$ converges to $\mathrm{supp}\;q_0$, \textit{i.e.},
\begin{align}
\lim_{k\rightarrow+\infty,t\rightarrow0}\mathrm{dist}(\mathbf y_k)=0.
\end{align}
\end{theorem}

\vspace{-5pt}

Now we formally propose \method, which trains GAN with an extra \textit{score matching regularity} from pre-trained DPM in a plug-in sense.
Let $g_\phi$ be the generator, we design a regularization term as below:
\begin{align}\label{eq:diffusion_loss}
\mathcal L_{score}=\mathbb E_{\mathbf z,\boldsymbol\epsilon,t}[\|\boldsymbol\epsilon_\theta(\alpha_tg_\phi(\mathbf z)+\sigma_t\boldsymbol\epsilon,t)-\boldsymbol\epsilon\|_2^2].
\end{align}
With a loss weight $\lambda_{score}$, the total objective of generator turns out to be $\mathcal L_G+\lambda_{score}\mathcal L_{score}$.
One can easily see that
\begin{align}
&\|\mathcal R(\mathbf x,\boldsymbol\epsilon,t)-\mathbf x\|_2^2\propto\|\boldsymbol\epsilon_\theta(\alpha_t\mathbf x+\sigma_t\boldsymbol\epsilon,t)-\boldsymbol\epsilon\|_2^2, \\
&\mathcal L_{score}\propto\mathbb E_{\mathbf z,\boldsymbol\epsilon,t}[\|\mathcal R(g_\phi(\mathbf z),\boldsymbol\epsilon,t)-g_\phi(\mathbf z)\|_2^2].\label{eq:equiv}
\end{align}

According to \Cref{thm:score_matching}, the optimality of \cref{eq:diffusion_loss} is equivalent to \textit{all generated samples supporting on the data manifold}.
However, such an optimality requires infinitely many one-step refinements with infinitesimal $t$, which is impractical during GAN training.
Therefore, practically we implement \method via \textit{relaxation} by employing finitely many one-step refinements with a relatively small $t$ instead, as shown in \cref{eq:equiv}.
On the other hand, \cref{eq:equiv} indicates that the score matching regularity aims to narrow the distance between synthesized samples and data manifold.

Under this circumstance, when generator distribution has positive-measure difference set over data manifold (indicating gradient vanishing), for each out-of-data-manifold sample $\mathbf x$, $\|\mathcal R(\mathbf x,\boldsymbol\epsilon,t)-\mathbf x\|_2^2$ remains positive, and $\mathcal L_{score}$ provides gradient for generator persistently to guide $\mathbf x$ to lie on data manifold.
Once all generated samples support on the data manifold, $\mathcal L_{score}$ will achieve optimality, and thus gradient will be annihilated.
In this case, the gradient vanishing issue can be largely mitigated, and \cref{eq:g_loss} will resume to supervise GAN training guaranteed by \Cref{thm:mode_collapse}.  
This profound conclusion facilitates GAN training from a novel perspective.

As an additional objective for generator, we further show the convergence and robustness of \method theoretically.
Recall that generator loss reaches its optimality if and only if the generator distribution coincides with the real distribution almost everywhere, indicating that all generated samples support on the data manifold (equivalent to the optimality of \cref{eq:diffusion_loss}).
In other words, once the generator loss is optimal, $\mathcal L_{score}$ will also achieves optimality, indicating the convergence and robustness.
Quantitative results of mean and variance are reported in \cref{tab:comparison_var}.

\begin{figure*}[!ht]
\centering
\includegraphics[width=1.0\textwidth]{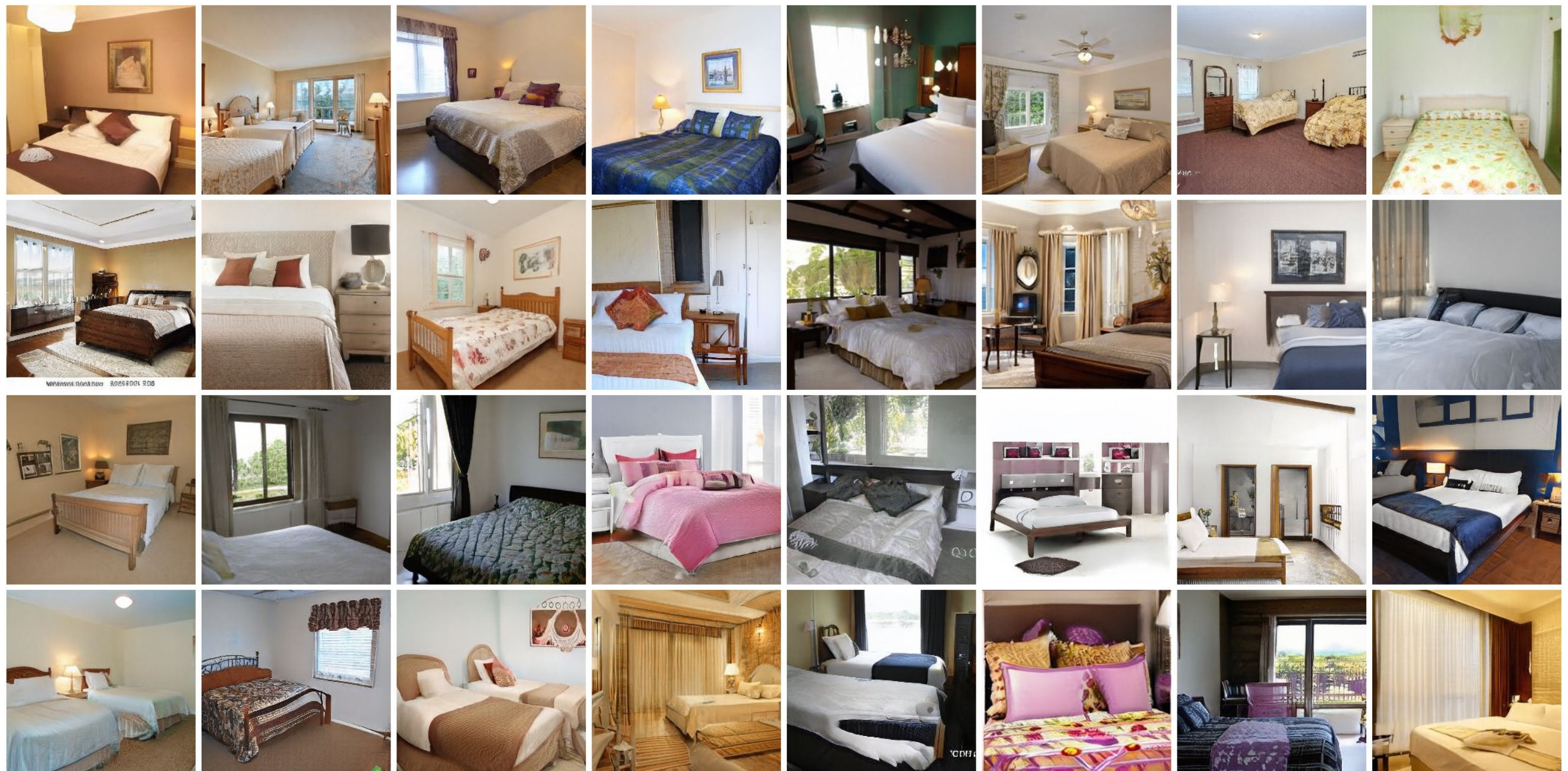}
\vspace{-12pt}
\caption{
    \textbf{Diverse results} generated by \method upon StyleGAN2~\cite{Karras2019AnalyzingAI} trained on LSUN Bedroom 256x256 dataset~\cite{yu15lsun}.
    We randomly sample the global latent code $\mathbf z$ for each image.
}
\label{fig:bedroom}
\vspace{-0pt}
\end{figure*}

It is also noteworthy that, serving as a relaxation, hyper-parameters are attached great importance to the efficacy of \method.
For instance, small $\lambda_{score}$ suggests inconspicuous guidance, weakening the functionality of \method.
However, when facing discrete data distribution, too strong regularity may restrict the generator distribution on only few modes, indicating that large $\lambda_{score}$ may affect synthesis diversity.
Detailed ablation study of $\lambda_{score}$ is addressed in \cref{tab:ablation_all}.

To take a further step, \method can be generalized to conditional GANs, in which GAN loss becomes:
\begin{align}
\mathcal L_G&=-\mathbb E_{\mathbf z,c}[\log D(G(\mathbf z,c),c)], \label{eq:cg_loss} \\
\mathcal L_D&=-\mathbb E_{\mathbf x,c}[\log D(\mathbf x,c)] \nonumber\\
&\qquad-\mathbb E_{\mathbf z,c}[\log(1-D(G(\mathbf z,c),c))], \label{eq:cd_loss}
\end{align}
where $c$ is the input condition.
To supervise the conditional GANs using \method, we simply add score matching regularity with a conditioned DPM as below:
\begin{align}{\hspace{-5pt}}\label{eq:cdiffusion_loss}
\mathcal L_{score}=\mathbb E_{\mathbf z,\boldsymbol\epsilon,t,c}[\|\boldsymbol\epsilon_\theta(\alpha_tg_\phi(\mathbf z,c)+\sigma_t\boldsymbol\epsilon,c,t)-\boldsymbol\epsilon\|_2^2].
\end{align}
We provide a theorem similar to \Cref{thm:score_matching} confirming the feasibility of \method under conditional generation settings, which is addressed in \cref{subsec:a.4}.

\subsection{Training Strategy}\label{subsec:3.4}

As a supernumerary regularity involved time-consuming DPM, it might be challenging to efficiently and effectively plug \method in native GAN training.
We propose the lazy strategy and narrowed timestep interval.
It is noteworthy that, even though our approach adopts the mechanisms of both adversarial learning and score matching regularity, there is no instability in the entire training process.

\begin{figure*}[!ht]
\centering
\includegraphics[width=1.0\textwidth]{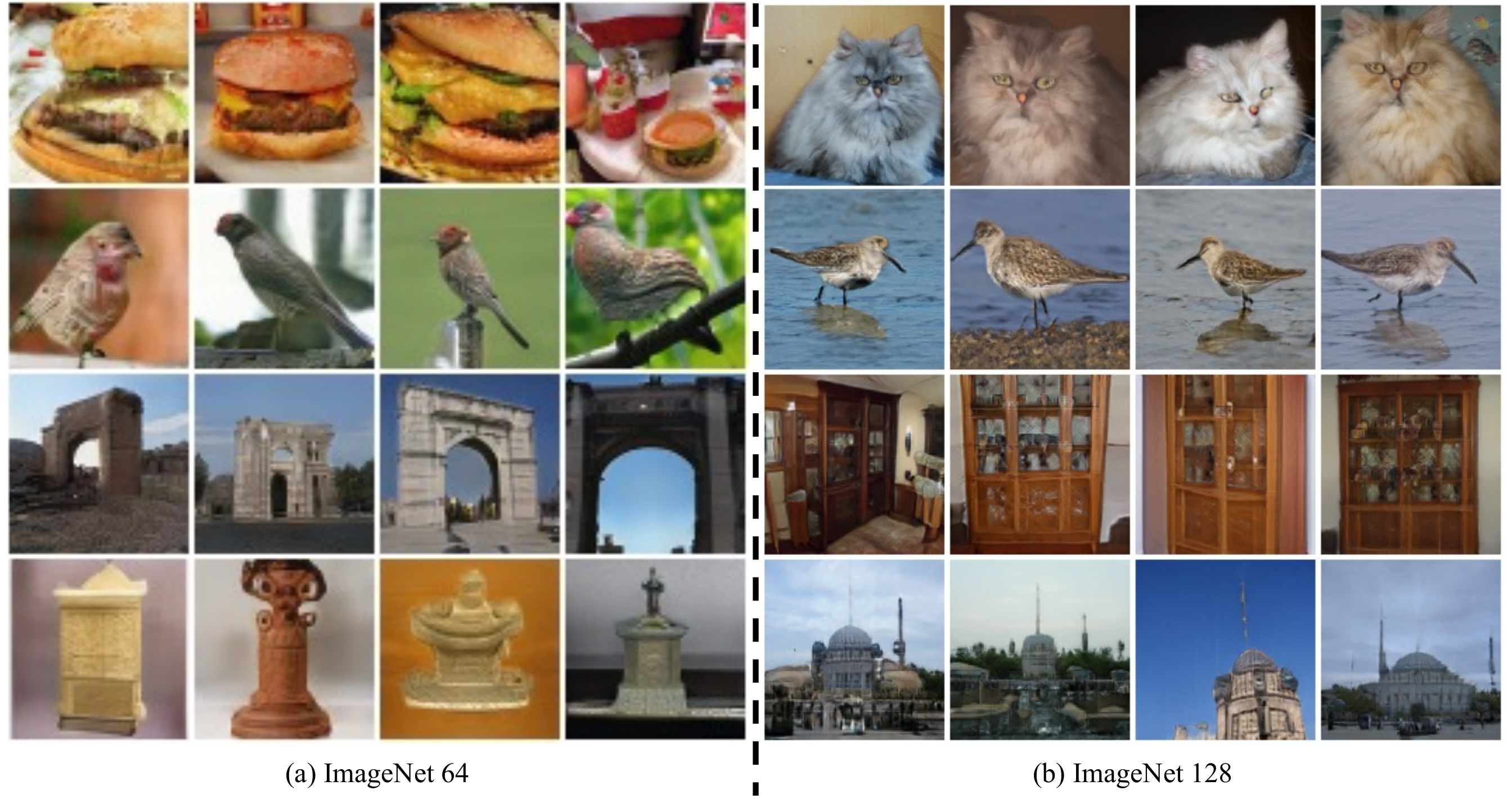}
\vspace{-18pt}
\caption{
    \textbf{Diverse results} generated by \method upon (a) Aurora~\cite{zhu2023aurora} on ImageNet 64x64 dataset~\cite{dengjia2009} and (b) BigGAN~\cite{Brock2018LargeSG} on ImageNet 128x128 dataset~\cite{dengjia2009}.
    We randomly sample four global latent codes $\mathbf z$ for each label condition $c$, demonstrated in each row.
}
\label{fig:imagenet}
\vspace{-2pt}
\end{figure*}

\noindent\textbf{Lazy strategy.}
We employ \textit{lazy strategy}~\cite{Karras2019AnalyzingAI} for \method, which applies the regularity less frequently than the main loss function, thus greatly diminishing the DPM computational cost.
\cref{tab:ablation_all} studies the efficacy of the regularity under different frequencies, providing an empirically adequate strategy.

\noindent\textbf{Narrowed timestep interval.}
Recall that score matching regularity can be considered as guidance from DDIM refinement.
Therefore, the involved timestep is attached great importance to the refinement performance.
Theoretically, large timestep suggests large discretization step of the differential equation, harming the quality of the refinement.
On the other hand, finite refinement steps entail that tiny timestep leads to inconspicuous refinement, since $\frac{\sigma_t}{\alpha_t}$ in \cref{eq:refine} tends to zero.
Performance comparison among different timestep intervals is addressed in \cref{tab:ablation_all}.

\definecolor{mygreen}{RGB}{34,170,133}
\begin{table*}[!ht]
\begin{minipage}[t]{0.48\textwidth}
\caption{
    \textbf{Sample quality} on CIFAR10~\cite{Krizhevsky_2009_17719}.
    $^*$Methods that require synthetic data construction for distillation.
    For clearer demonstration, one-step approaches including GANs and DPMs are highlighted in \textbf{\textcolor{gray}{gray}}.
}
\label{tab:cifar}
\vskip 0.15in
\centering
\SetTblrInner{rowsep=0.548pt}               
\SetTblrInner{colsep=3.0pt}                
\scriptsize
\begin{tblr}{
    cell{1-39}{2-4}={halign=c,valign=m},   
    cell{1-39}{1}={halign=l,valign=m},     
    hline{1,2,40}={1-4}{1.0pt},            
    hline{19,32}={1-4}{},                  
    cell{19-39}{1-4}={bg=lightgray!35},
}
METHOD                                            & NFE $(\downarrow)$ & FID $(\downarrow)$ & IS $(\uparrow)$ \\
ScoreSDE~\cite{song2020score}                     &               2000 &               2.20 &        \bf 9.89 \\
DDPM~\cite{ho2020denoising}                       &               1000 &               3.17 &            9.46 \\
LSGM~\cite{vahdat2021score}                       &                147 &               2.10 &              -- \\
PFGM~\cite{xu2022poisson}                         &                110 &               2.35 &            9.68 \\
EDM~\cite{Karras2022edm}                          &                 35 &           \bf 1.97 &              -- \\
DDIM~\cite{song2020denoising}                     &                 50 &               4.67 &              -- \\
DDIM~\cite{song2020denoising}                     &                 30 &               6.84 &              -- \\
DDIM~\cite{song2020denoising}                     &                 10 &               8,23 &              -- \\
DPM-solver-3~\cite{lu2022dpm}                     &                 12 &               6.03 &              -- \\
3-DEIS~\cite{zhang2022fast}                       &                 10 &               4.17 &              -- \\
UniPC~\cite{zhao2023unipc}                        &                  8 &               5.10 &              -- \\
UniPC~\cite{zhao2023unipc}                        &                  5 &              23.22 &              -- \\
Denoise Diffusion GAN (T=2)~\cite{xiao2022DDGAN}  &                  2 &               4.08 &            9.80 \\
PD~\cite{SalimansH22}                             &                  2 &               5.58 &            9.05 \\
CT~\cite{song2023consistency}                     &                  2 &               5.83 &            8.85 \\
iCT~\cite{song2023improved}                       &                  2 &               2.46 &            9.80 \\
CD~\cite{song2023consistency}                     &                  2 &               2.93 &            9.75 \\
Denoise Diffusion GAN (T=1)~\cite{xiao2022DDGAN}  &              \bf 1 &              14.60 &            8.93 \\
KD$^*$~\cite{luhman2021knowledge}                 &              \bf 1 &               9.36 &              -- \\
TDPM~\cite{zheng2022truncated}                    &              \bf 1 &               8.91 &            8.65 \\
1-ReFlow~\cite{liu2022flow}                       &              \bf 1 &             378.00 &            1.13 \\
CT~\cite{song2023consistency}                     &              \bf 1 &               8.70 &            8.49 \\
iCT~\cite{song2023improved}                       &              \bf 1 &           \bf 2.83 &        \bf 9.54 \\
1-ReFlow (+distill)$^*$~\cite{liu2022flow}        &              \bf 1 &               6.18 &            9.08 \\
2-ReFlow (+distill)$^*$~\cite{liu2022flow}        &              \bf 1 &               4.85 &            9.01 \\
3-ReFlow (+distill)$^*$~\cite{liu2022flow}        &              \bf 1 &               5.21 &            8.79 \\
PD~\cite{SalimansH22}                             &              \bf 1 &               8.34 &            8.69 \\
CD-L2~\cite{song2023consistency}                  &              \bf 1 &               7.90 &              -- \\
CD-LPIPS~\cite{song2023consistency}               &              \bf 1 &               3.55 &            9.48 \\
Diff-Instruct~\cite{luo2023diffinstruct}          &              \bf 1 &               4.19 &              -- \\
AutoGAN~\cite{Gong_2019_ICCV}                     &              \bf 1 &              12.40 &            8.55 \\
E2GAN~\cite{Tian_2020_ECCV}                       &              \bf 1 &              11.30 &            8.51 \\
TransGAN~\cite{jiang2021transgan}                 &              \bf 1 &               9.26 &            9.05 \\
StyleGAN-XL~\cite{Sauer2021ARXIV}                 &              \bf 1 &           \bf 1.85 &              -- \\
Diffusion StyleGAN2~\cite{wang2022diffusiongan}   &              \bf 1 &               3.19 &            9.94 \\
StyleGAN2-ADA~\cite{Karras2019AnalyzingAI}        &              \bf 1 &               2.42 &           10.14 \\
StyleGAN2-ADA+Tune+DI~\cite{luo2023diffinstruct}  &              \bf 1 &               2.27 &           10.11 \\
StyleGAN2-ADA \textcolor{mygreen}{\bf + \method}  &              \bf 1 &               2.06 &       \bf 10.22 \\
\end{tblr}
\end{minipage}
\hfill
\begin{minipage}[t]{0.48\textwidth}
\caption{
    \textbf{Sample quality} on ImageNet~\cite{dengjia2009} and LSUN Bedroom 256x256~\cite{yu15lsun}.
    $^\dag$Methods that utilize distillation techniques.
    $^\ddag$Methods that are trained by ourselves with official implementation.
    For clearer demonstration, one-step approaches including GANs and DPMs are highlighted in \textbf{\textcolor{gray}{gray}}.
}
\label{tab:imagenet_lsun}
\vskip 0.15in
\centering
\SetTblrInner{rowsep=0.5pt}      
\SetTblrInner{colsep=1.5pt}      
\scriptsize
\begin{tblr}{
    cell{1-36}{2-5}={halign=c,valign=m},         
    cell{1-36}{1}={halign=l,valign=m},           
    hline{1-3,18,19,22,23,37}={1-5}{1.0pt},      
    hline{7,14,20,27,32}={1-5}{},                
    cell{5,6,12-17,20,21,25,26,31-36}{1-5}={bg=lightgray!35},
}
METHOD                                            & NFE $(\downarrow)$ & FID $(\downarrow)$ & Prec. $(\uparrow)$ & Rec. $(\uparrow)$ \\
\textbf{ImageNet 64x64}                           &                    &                    &                    &                   \\[1.2pt]
PD$^\dag$~\cite{SalimansH22}                      &                  2 &               8.95 &               0.63 &          \bf 0.65 \\
CD$^\dag$~\cite{song2023consistency}              &                  2 &           \bf 4.70 &           \bf 0.69 &              0.64 \\
PD$^\dag$~\cite{SalimansH22}                      &              \bf 1 &              15.39 &               0.59 &              0.62 \\
CD$^\dag$~\cite{song2023consistency}              &              \bf 1 &               6.20 &               0.68 &              0.63 \\
ADM~\cite{dhariwal2021diffusion}                  &                250 &           \bf 2.07 &           \bf 0.74 &              0.63 \\
EDM~\cite{Karras2022edm}                          &                 79 &               2.44 &               0.71 &          \bf 0.67 \\
DDIM~\cite{song2020denoising}                     &                 50 &              13.70 &               0.65 &              0.56 \\
DEIS~\cite{zhang2022fast}                         &                 10 &               6.65 &                 -- &                -- \\
CT~\cite{song2023consistency}                     &                  2 &              11.10 &               0.69 &              0.56 \\
CT~\cite{song2023consistency}                     &              \bf 1 &              13.00 &               0.71 &              0.47 \\
iCT~\cite{song2023improved}                       &              \bf 1 &               4.02 &               0.70 &              0.63 \\
StyleGAN2~\cite{Karras2019AnalyzingAI}    &              \bf 1 &              21.32 &               0.42 &              0.36 \\
StyleGAN2 \textcolor{mygreen}{\bf + \method}      &              \bf 1 &              18.31 &           \bf 0.45 &              0.39 \\
Aurora$^\ddag$~\cite{zhu2023aurora}               &              \bf 1 &               8.87 &               0.41 &              0.48 \\
Aurora \textcolor{mygreen}{\bf + \method}         &              \bf 1 &           \bf 7.11 &               0.42 &          \bf 0.49 \\
\textbf{ImageNet 128x128}                         &                    &                    &                    &                   \\[1.2pt]
ADM~\cite{dhariwal2021diffusion}                  &                250 &               5.91 &               0.70 &              0.65 \\
BigGAN$^\ddag$~\cite{Brock2018LargeSG}            &              \bf 1 &              10.76 &               0.73 &              0.29 \\
BigGAN \textcolor{mygreen}{\bf + \method}         &              \bf 1 &           \bf 9.49 &           \bf 0.77 &          \bf 0.30 \\
\textbf{LSUN Bedroom 256x256}                     &                    &                    &                    &                   \\[1.2pt]
PD$^\dag$~\cite{SalimansH22}                      &                  2 &               8.47 &               0.56 &          \bf 0.39 \\
CD$^\dag$~\cite{song2023consistency}              &                  2 &           \bf 5.22 &           \bf 0.68 &          \bf 0.39 \\
PD$^\dag$~\cite{SalimansH22}                      &              \bf 1 &              16.92 &               0.47 &              0.27 \\
CD$^\dag$~\cite{song2023consistency}              &              \bf 1 &               7.80 &               0.66 &              0.34 \\
DDPM~\cite{ho2020denoising}                       &               1000 &               4.89 &               0.60 &              0.45 \\
ADM~\cite{dhariwal2021diffusion}                  &               1000 &           \bf 1.90 &               0.66 &          \bf 0.51 \\
EDM~\cite{Karras2022edm}                          &                 79 &               3.57 &               0.66 &              0.45 \\
CT~\cite{song2023consistency}                     &                  2 &               7.85 &           \bf 0.68 &              0.33 \\ 
CT~\cite{song2023consistency}                     &              \bf 1 &              16.00 &               0.60 &              0.17 \\
PGGAN~\cite{Karras2017ProgressiveGO}              &              \bf 1 &               8.34 &                 -- &                -- \\
PG-SWGAN~\cite{wu19pgsw}                          &              \bf 1 &               8.00 &                 -- &                -- \\
StyleGAN2~\cite{Karras2019AnalyzingAI}            &              \bf 1 &               2.35 &               0.59 &              0.48 \\
Diffusion StyleGAN2~\cite{wang2022diffusiongan}   &              \bf 1 &               3.65 &               0.60 &              0.32 \\
StyleGAN2 \textcolor{mygreen}{\bf + \method}      &              \bf 1 &           \bf 1.98 &           \bf 0.61 &          \bf 0.49 \\
\end{tblr}
\end{minipage}
\vspace{-5pt}
\end{table*}

\section{Experiments}\label{sec:exp}

\subsection{Experimental Setups}\label{subsec:4.1}

\noindent\textbf{Datasets and baselines.}
We apply \method to previous seminal GANs, including StyleGAN2~\cite{Karras2019AnalyzingAI}, BigGAN~\cite{Brock2018LargeSG}, and Aurora~\cite{zhu2023aurora}.
We train StyleGAN2 on CIFAR10~\citep{Krizhevsky_2009_17719}, ImageNet 64x64~\cite{dengjia2009}, and LSUN Bedroom 256x256~\cite{yu15lsun}.
Additionally, we train BigGAN and Aurora on ImageNet 128x128 and 64x64~\cite{dengjia2009}, respectively.

\noindent\textbf{Evaluation metrics.}
We draw 50,000 samples for Fr\'{e}chet Inception Distance (FID)~\cite{heusel2017gans} to evaluate the fidelity of the synthesized images.
Inception Score (IS)~\cite{Salimans2016ImprovedTF} measures how well a model captures the full ImageNet class distribution while still convincingly producing individual samples from a single class.
Finally, we use Improved Precision (Prec.) and Recall (Rec.)~\cite{Kynknniemi2019ImprovedPA} to separately measure sample fidelity (Precision) and diversity (Recall).

\noindent\textbf{Implementation details.}
We train \method with NVIDIA A100 GPUs.
With abundant powerful pre-trained DPMs as expertise, we choose the state-of-the-art pre-trained ADM~\cite{dhariwal2021diffusion}\footnote{https://github.com/openai/guided-diffusion} and EDM~\cite{Karras2022edm}\footnote{https://github.com/NVlabs/edm} provided in the official implementation.
Regarding GANs, we use the third-party implementation of StyleGAN2\footnote{https://github.com/bytedance/Hammer}~\cite{Karras2019AnalyzingAI} on CIFAR10, ImageNet, and LSUN Bedroom under Hammer~\cite{hammer2022} (official results on CIFAR10 and LSUN Bedroom are reported) and officially implemented BigGAN\footnote{https://github.com/ajbrock/BigGAN-PyTorch}~\cite{Brock2018LargeSG} and Aurora\footnote{https://github.com/zhujiapeng/Aurora}~\cite{zhu2023aurora}.

\subsection{Toy Example on Self-designed Dataset}\label{subsec:4.2}

We conduct experiments of generation task on the discrete data distribution.
The toy data is simulated by a mixture of 49 2-dimensional Gaussian distributions with extremely low variance.
Following~\citet{wang2022diffusiongan}, we train a small GAN model, whose generator and discriminator are both parameterized by MLPs, with two 128-unit hidden layers and Leaky ReLU activation functions.
The training results are shown in \cref{fig:toy_model}.
Note that the vanilla GAN exhibits poor synthesis discreteness.
By adopting the noise injection to discriminator, DiffusionGAN~\cite{wang2022diffusiongan} turns to fit the distribution of noisy data, endeavoring to promote synthesis diversity.
However, this compromises the synthesis quality to a certain extent, making the generated samples less discrete.
As a comparison, our \method is capable of capturing the discrete distribution, confirming the feasibility by simply adding the score matching regularity.
This indicates that \method manages to alleviate the gradient vanishing by eliminating out-of-data-manifold samples.

\subsection{Results on Real Datasets}\label{subsec:4.3}

\noindent\textbf{Qualitative results.}
We showcase some results in \cref{fig:bedroom,fig:imagenet}.
One can see that, with score matching regularity, GAN is more capable of synthesizing samples addressed on data manifold, especially the conditional generation in \cref{fig:imagenet}, getting out of the dilemma of gradient vanishing.
It is also noteworthy that \method promotes the synthesis diversity to a certain extent, since generator loss provides more significant guidance on the data manifold.

\vspace{3.5pt}

\noindent\textbf{Quantitative comparison.}
Besides the exhibited qualitative results, we also provide quantitative comparison between baseline and \method-improving version on various state-of-the-art GANs, conveying an overall picture of its capability of promoting generation performance.
In \cref{tab:cifar,tab:imagenet_lsun}, we report the evaluation results on three different data domains, including CIFAR10, LSUN Bedroom 256x256, and ImageNet.
We can tell that \method achieves performance improvement on the three datasets.

\begin{table*}[!ht]
\caption{
    \textbf{Ablation study} of frequency of lazy strategy, narrowed timestep interval, and loss weight $\lambda_{score}$ using StyleGAN2-ADA~\cite{Karras2019AnalyzingAI} on CIFAR10~\cite{Krizhevsky_2009_17719}.
    The four values in each cell represent the evaluation metric with respect to $\lambda_{score}=0.01,0.025,0.05,0.1$, respectively.
}
\label{tab:ablation_all}
\vspace{-5pt}
\vskip 0.15in
\centering
\footnotesize
\SetTblrInner{rowsep=1.5pt}            
\SetTblrInner{colsep=6.0pt}            
\begin{tblr}{
    cells={halign=c,valign=m},         
    hline{2,5,8,11,14,17}={},    
    hline{1,4,7,10,13,16,19}={1.0pt},  
    vline{2}={},               
}
    Freq. $=4$         &            $t\in[5,15]$ &           $t\in[40,60]$ &          $t\in[90,110]$ &         $t\in[225,275]$ \\
    FID ($\downarrow$) &     2.30/2.32/2.33/2.37 &     2.24/2.26/2.28/2.31 &     2.24/2.28/2.31/2.33 &     2.25/2.29/2.33/2.32 \\
    IS ($\uparrow$)    & 10.16/10.15/10.15/10.14 & 10.19/10.21/10.18/10.17 & 10.19/10.18/10.17/10.16 & 10.21/10.20/10.20/10.21 \\
    Freq. $=8$         &            $t\in[5,15]$ &           $t\in[40,60]$ &          $t\in[90,110]$ &         $t\in[225,275]$ \\
    FID ($\downarrow$) &     2.17/2.22/2.17/2.18 &     2.06/2.08/2.07/2.08 &     2.09/2.10/2.12/2.12 &     2.10/2.09/2.10/2.11 \\
    IS ($\uparrow$)    & 10.15/10.14/10.16/10.15 & 10.22/10.20/10.20/10.20 & 10.20/10.17/10.21/10.18 & 10.24/10.25/10.23/10.26 \\
    Freq. $=16$        &            $t\in[5,15]$ &           $t\in[40,60]$ &          $t\in[90,110]$ &         $t\in[225,275]$ \\
    FID ($\downarrow$) &     2.15/2.20/2.18/2.26 &     2.07/2.11/2.15/2.19 &     2.10/2.11/2.19/2.20 &     2.16/2.15/2.16/2.19 \\
    IS ($\uparrow$)    & 10.16/10.15/10.17/10.15 & 10.22/10.20/10.18/10.17 & 10.22/10.22/10.19/10.18 & 10.20/10.21/10.21/10.21 \\
    Freq. $=32$        &            $t\in[5,15]$ &           $t\in[40,60]$ &          $t\in[90,110]$ &         $t\in[225,275]$ \\
    FID ($\downarrow$) &     2.18/2.20/2.22/2.23 &     2.17/2.17/2.19/2.19 &     2.18/2.18/2.19/2.23 &     2.22/2.24/2.24/2.27 \\
    IS ($\uparrow$)    & 10.17/10.16/10.17/10.16 & 10.21/10.21/10.19/10.20 & 10.19/10.21/10.21/10.19 & 10.21/10.20/10.19/10.18 \\
    Freq. $=64$        &            $t\in[5,15]$ &           $t\in[40,60]$ &          $t\in[90,110]$ &         $t\in[225,275]$ \\
    FID ($\downarrow$) &     2.20/2.23/2.27/2.25 &     2.19/2.20/2.22/2.24 &     2.19/2.27/2.25/2.27 &     2.28/2.25/2.26/2.34 \\
    IS ($\uparrow$)    & 10.18/10.16/10.16/10.17 & 10.20/10.20/10.19/10.18 & 10.21/10.19/10.20/10.18 & 10.18/10.18/10.17/10.17 \\
    Freq. $=128$       &            $t\in[5,15]$ &           $t\in[40,60]$ &          $t\in[90,110]$ &         $t\in[225,275]$ \\
    FID ($\downarrow$) &     2.22/2.24/2.28/2.29 &     2.20/2.23/2.24/2.26 &     2.22/2.28/2.30/2.29 &     2.28/2.29/2.33/2.35 \\
    IS ($\uparrow$)    & 10.16/10.15/10.15/10.25 & 10.19/10.19/10.17/10.15 & 10.18/10.16/10.14/10.15 & 10.17/10.15/10.16/10.14 \\
\end{tblr}
\vspace{-10pt}
\end{table*}

\subsection{Analyses}\label{subsec:4.4}

\begin{table*}[!ht]
\caption{
    \textbf{Quantitative results} measured by FID~($\downarrow$), Precision~($\uparrow$), and Recall~($\uparrow$) on ImageNet 64x64.
    We report the mean and variance of evaluation metrics with 5 independent sampling.
}
\label{tab:comparison_var}
\vspace{-10pt}
\vskip 0.15in
\centering
\footnotesize
\SetTblrInner{rowsep=1.0pt}      
\SetTblrInner{colsep=22.0pt}     
\begin{tblr}{
    cell{1-4}{2-5}={halign=c,valign=m},         
    cell{1-4}{1}={halign=l,valign=m},           
    hline{1,5}={1-5}{1.0pt},                    
    hline{2}={1-5}{},                           
}
Method                      & FID ($\downarrow$) & Prec. ($\uparrow$) & Rec. ($\uparrow$) & \# GPUs \\
Aurora~\cite{zhu2023aurora} &               8.87 &               0.41 &              0.48 &      16 \\
Aurora + \method            &    7.74 $\pm$ 0.16 &    0.41 $\pm$ 0.00 &   0.49 $\pm$ 0.00 &      16 \\
Aurora + \method            &    7.16 $\pm$ 0.05 &    0.42 $\pm$ 0.00 &   0.49 $\pm$ 0.00 &      32 \\
\end{tblr}
\vspace{-10pt}
\end{table*}

\begin{table*}[!ht]
\caption{
    \textbf{Comparison} of computational cost on Aurora~\cite{zhu2023aurora} and BigGAN~\cite{Brock2018LargeSG} on ImageNet~\cite{dengjia2009} with resolution 64 and 128, respectively.
    We report FID~($\downarrow$), GPU memory \textit{w/} and \textit{w/o} regularity, respectively.
}
\label{tab:memory_cost}
\vspace{-5pt}
\vskip 0.15in
\centering
\footnotesize
\SetTblrInner{rowsep=1.0pt}      
\SetTblrInner{colsep=8.0pt}      
\begin{tblr}{
    cell{1-6}{2-4}={halign=c,valign=m},         
    cell{1-6}{1}={halign=l,valign=m},           
    hline{1,7}={1-4}{1.0pt},                    
    hline{2}={1-4}{},                           
}
Method                                     & FID ($\downarrow$) & GPU Memory \textit{w/o} Regularity & GPU Memory \textit{w/} Regularity \\
Aurora~\cite{zhu2023aurora}                &               8.87 &                           36.10 GB &                               N/A \\
Aurora + \method                           &               7.11 &                           36.10 GB &                          56.00 GB \\
Aurora + \method (omitting U-Net Jacobian) &               7.34 &                           36.10 GB &                          37.22 GB \\
BigGAN~\cite{Brock2018LargeSG}             &              10.76 &                           16.73 GB &                               N/A \\
BigGAN + \method                           &               9.49 &                           16.73 GB &                          21.67 GB \\
\end{tblr}
\vspace{-0pt}
\end{table*}

\begin{figure*}[!ht]
\centering
\includegraphics[width=1.0\textwidth]{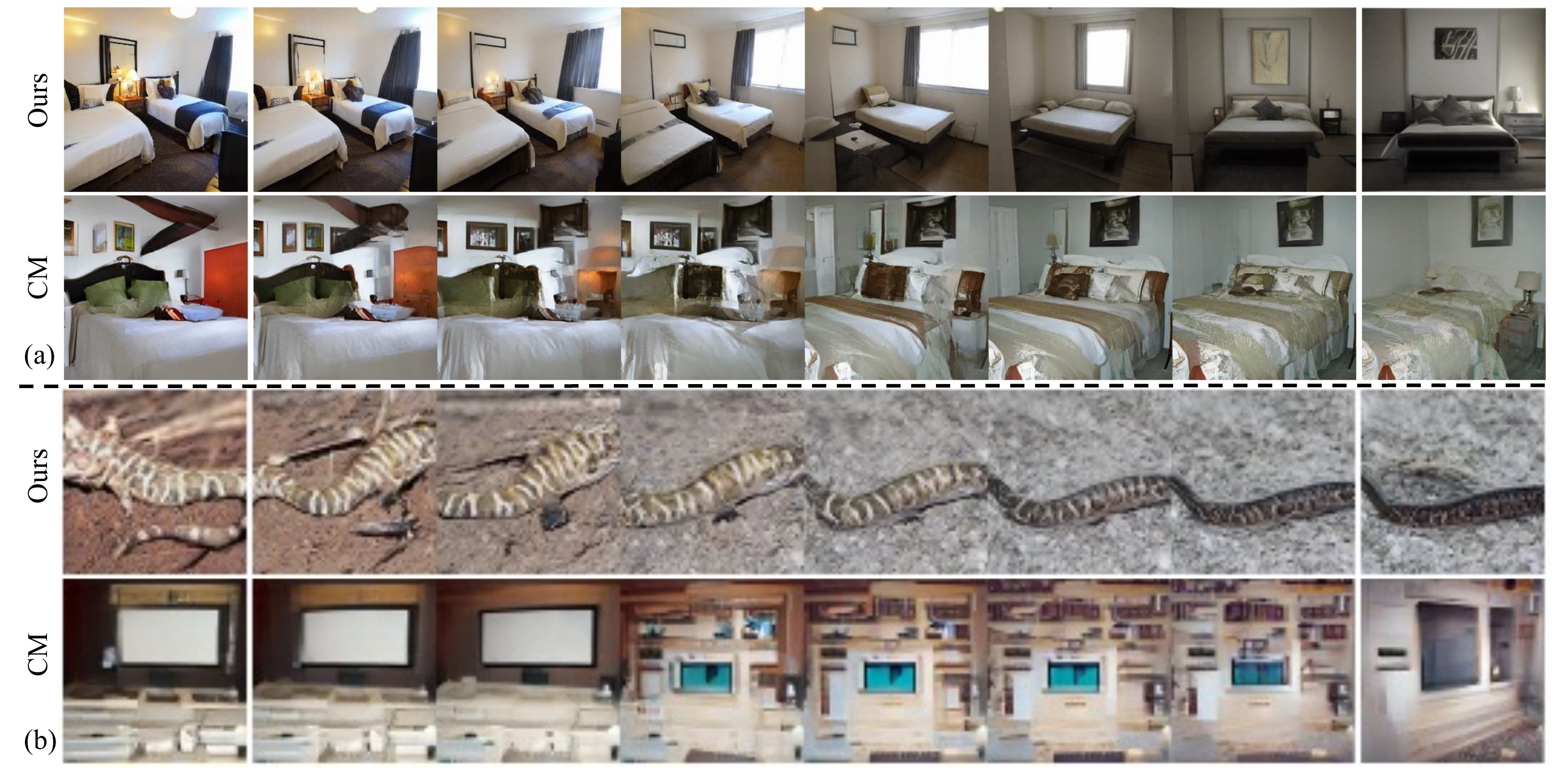}
\vspace{-20pt}
\caption{%
    \textbf{Visualization} of latent interpolation results on (a) LSUN Bedroom 256x256, and (b) ImageNet 64x64.
    We employ StyleGAN2~\cite{Karras2019AnalyzingAI} and Consistency Model (CM)~\cite{song2023consistency} on LSUN Bedroom 256x256 dataset, interpolating in the disentangled latent space.
    As for interpolation on ImageNet 64x64 dataset, we introduce Aurora~\cite{zhu2023aurora} and CM~\cite{song2023consistency}.
    We fix the label condition $c$, and only interpolate in the disentangled latent space $\mathcal W$.
    It is noteworthy that both StyleGAN2 and Aurora are strongly capable of synthesizing correct interpolation results, due to the extremely smooth and well-studied latent spaces.
    However, CM fails to generate interpolation results, due to poor semantic continuity in the latent space.
}
\vspace{-1pt}
\label{fig:interp}
\end{figure*}

\noindent\textbf{Convergence and robustness.}
Recall that \method alters the objective for generator and employs supervision involving random noise.
This may lead to potential non-robustness or convergence issue.
We compute the mean and variance of evaluation metrics with 5 independent sampling.
\cref{tab:comparison_var} confirms the robustness of \method convincingly.
Besides, abundant ablation studies in \cref{tab:ablation_all} also confirm the efficacy across a large range of timesteps $t$ and frequencies.

\noindent\textbf{Computational cost comparison.}
As one of the representative one-step generation paradigms, Consistency Model~\cite{song2023consistency} achieves satisfactory performance while consuming 64 GPUs.
As a comparison, we report in \cref{tab:memory_cost,tab:time_cost} the FID performance, average iteration time, GPU memory, and number of used GPUs, respectively.
Note that we adopt the lazy regularization strategy in \method, where the diffusion model is involved only once per 8 iterations.
Hence, for most training steps, the GPU memory is not increased.
We further employ the trick to omit the U-Net Jacobian term for more efficient gradient optimization following~\citet{poole2023dreamfusion}, further alleviating the memory explosion while not sacrificing the performance.
We can tell that, our \method slightly increases the training cost but significantly improve the performance.

\noindent\textbf{Latent interpolation.}
Latent space interpolation is widely studied in the seminal literature~\cite{Brock2018LargeSG,Karras2018ASG,Karras2019AnalyzingAI}, which aims to verify the generative ability of the GANs.
It is well recognized that GANs possess semantically continuous and extremely smooth latent spaces~\cite{gansteerability,wu2021,Shen2019InterpretingTL}.
We demonstrate the results of latent interpolation within $\mathcal W$ space (\textit{i.e.}, the disentangled latent space) in \cref{fig:interp}, in which the observation coincides with the common conclusion.
As a comparison, we show the interpolation of Consistency Model (CM)~\cite{song2023consistency}, which is also a one-step synthesis paradigm.
As shown in \cref{fig:interp}, CM fails to synthesize correct results with interpolated latent codes, indicating the poor continuity of the latent space and thus the difficulty for editing and other downstream applications.

\noindent\textbf{Ablation study.}
We conduct comprehensive ablation studies to convey a direct and clear picture of the efficacy of the score matching regularity under different settings, as reported in \cref{tab:ablation_all}.
We can conclude that both too large (\textit{e.g.}, $t\in[225, 275]$) and too tiny (\textit{e.g.}, $t\in[5,15]$) timestep interval negatively influences the synthesis performance.
Besides, large $\lambda_{score}$ harms both FID and IS performance.
Finally, too frequent regularization also harms the performance while too infrequent regularity indicates inconspicuous improvements.
It is noteworthy that all experimental results coincide with the analysis in \cref{subsec:3.4}.
Empirical value of these hyper-parameters used in our experiments are listed in \cref{subsec:b.1}.

\vspace{3.5pt}

\subsection{Discussion}\label{subsec:4.5}

It is the gradient vanishing of GANs that restricts the downstream applications, leaving GANs lacking further research such as text-to-image synthesis.
Therefore, we believe \method is attached to great importance.
Despite the great success on facilitating GAN training, our proposed algorithm has several potential limitations.
As a supplemental regularity, its efficacy depends highly on the choice of the hyper-parameters.
Although we conduct extensive and convincing ablation studies and provide an empirically adequate solution, the optimality of such a strategy is currently unexplored.
Besides, due to additionally involving the score matching via noise prediction model, we introduce the lazy strategy to diminish time cost.
However, this still slightly increases training cost and slows down the training speed.
Therefore, how to further conquer this problem (\textit{e.g.}, employing a smaller-resolution diffusion model) will be an interesting avenue for future research.
Although failing to outperform DPMs, we hope that \method will encourage the community to close the gap in the future.

\section{Conclusion}\label{sec:conclusion}

In this paper, we analyze and alleviate gradient vanishing of GANs by delving into the mathematical foundation of GAN loss.
We theoretically point out a novel perspective to facilitate GAN training.
Drawing lessons from score matching, we propose \method, a plug-in algorithm which punishes gradient vanishing.
We provide a proof that score matching serving as a regularity provides supernumerary guidance enforcing out-of-data-manifold samples by generator towards data manifold.
Consequently, generator loss is more capable of guiding generator distribution to converge to data distribution.
We conduct comprehensive experiments to demonstrate significant improvement of synthesis quality on a variety of datasets and baseline models.

\section*{Acknowledgements}

This work was partially supported by Beijing Natural Science Foundation (L222008), the Natural Science Foundation of China (U2336214, 62332019, 62302297), Beijing Hospitals Authority Clinical Medicine Development of special funding support (ZLRK202330), Shanghai Sailing Program (22YF1420300), and Young Elite Scientists Sponsorship Program by CAST (2022QNRC001).

\section*{Impact Statement}

The proposed approach represents a significant advancement in the field of GAN synthesis, which is both a fundamental and practical problem.
The introduction of the score matching regularity from pre-trained DPMs has greatly enhanced the capacity for handling gradient vanishing, making GAN training more stable.
However, this technique may also facilitate the creation of fake content, such as Deepfake, which could have negative consequences.
We want to emphasize that we strongly oppose the misuse of this approach to violate security and privacy issues.
The negative impact of such misuse can be mitigated by the development of deep fake detection technology.

\bibliography{ref}
\bibliographystyle{icml2024}

\newpage
\appendix
\onecolumn
\unnumberedsection{Appendix}\label{sec:appendix}

\section{Proofs and derivations}

In this section, we will prove the theorems claimed in the main manuscript.
First, we emphasize a property in Riemann integral, which is attached great importance to the proofs and derivations in the sequel.

\begin{proposition}\label{prop:zero_integral}
Let $f$ be a Lebesgue-measurable function, and $Z$ be a set with zero Lebesgue measure.
The integral of $f$ on $Z$ is 0, \textit{i.e.},
\begin{align}
\int_Zf(\mathbf x)\mathrm d\mathbf x=0.
\end{align}
\end{proposition}

\subsection{Proof of \Cref{thm:nongrad}}\label{subsec:a.1}

\begin{theorem}
Let $A,B\subset\mathbb R^d$ with positive Lebesgue measure, \textit{i.e.}, $\mu_d(A)>0,\mu_d(B)>0$.
Denote by $q_A(\mathbf x),q_B(\mathbf x)$ two distributions supported on $A,B$, respectively, \textit{i.e.}, $\mathrm{supp}\;q_A=\{\mathbf x\mid q_A(\mathbf x)\neq0\}=A$, $\mathrm{supp}\;q_B=B$.
Let $X\backslash Y=\{\mathbf x\mid\mathbf x\in X\text{ and }\mathbf x\notin Y\}$.
When $D$ reaches the optimality, and if $\mu_d(A\backslash B)>0$, then
\begin{align}
-\int q_A(\mathbf x)\log\frac{q_B(\mathbf x)}{q_A(\mathbf x)+q_B(\mathbf x)}\mathrm d\mathbf x=+\infty.
\end{align}
\end{theorem}

\begin{proof}
We first divide the union of $A$ and $B$ as below:
\begin{align}
A\cup B=(A\backslash B)\coprod(B\backslash A)\coprod(A\cap B),
\end{align}
where $\coprod$ represents the disjoint union.
Note that $q_A(\mathbf x)\log\frac{q_B(\mathbf x)}{q_A(\mathbf x)+q_B(\mathbf x)}=0$ for $\mathbf x\notin A\cup B$ and $\mathbf x\in B\backslash A$.
Therefore
\begin{align}
&-\int q_A(\mathbf x)\log\frac{q_B(\mathbf x)}{q_A(\mathbf x)+q_B(\mathbf x)}\mathrm d\mathbf x \\
=&-\int_{A\cup B} q_A(\mathbf x)\log\frac{q_B(\mathbf x)}{q_A(\mathbf x)+q_B(\mathbf x)}\mathrm d\mathbf x \\
=&-\int_{A\backslash B} q_A(\mathbf x)\log\frac{q_B(\mathbf x)}{q_A(\mathbf x)+q_B(\mathbf x)}\mathrm d\mathbf x\nonumber \\
&\qquad-\int_{B\backslash A} q_A(\mathbf x)\log\frac{q_B(\mathbf x)}{q_A(\mathbf x)+q_B(\mathbf x)}\mathrm d\mathbf x\nonumber \\
&\qquad-\int_{A\cap B} q_A(\mathbf x)\log\frac{q_B(\mathbf x)}{q_A(\mathbf x)+q_B(\mathbf x)}\mathrm d\mathbf x \\
\geqslant&-\int_{A\backslash B} q_A(\mathbf x)\log\frac{q_B(\mathbf x)}{q_A(\mathbf x)+q_B(\mathbf x)}\mathrm d\mathbf x+0+0 \\
=&+\infty\label{eq:infty},
\end{align}
where $-\int_{A\cap B} q_A(\mathbf x)\log\frac{q_B(\mathbf x)}{q_A(\mathbf x)+q_B(\mathbf x)}\mathrm d\mathbf x\geqslant-\sup_{\mathbf x\in A\cap B}\log\frac{q_B(\mathbf x)}{q_A(\mathbf x)+q_B(\mathbf x)}(\int_{A\cap B} q_A(\mathbf x)\mathrm d\mathbf x)\geqslant0$ is by the property of Riemann integral.
And \cref{eq:infty} is due to $\mu_d(A\backslash B)>0$ and $\frac{q_B(\mathbf x)}{q_A(\mathbf x)+q_B(\mathbf x)}=0$ on $A\backslash B$.
\end{proof}

\subsection{Proof of \Cref{thm:mode_collapse}}\label{subsec:a.2}

\begin{theorem}
Following the settings in \Cref{thm:nongrad}, when $D$ reaches the optimality, the following inequality reaches its optimal if and only if $\mu_d(A\backslash B)=\mu_d(B\backslash A)=0$, and $\mu_d(\{\mathbf x\mid q_A|_{A\cap B}(\mathbf x)\neq q_B|_{A\cap B}(\mathbf x)\})=0$.
\begin{align}
-\int q_A(\mathbf x)\log\frac{q_B(\mathbf x)}{q_A(\mathbf x)+q_B(\mathbf x)}\mathrm d\mathbf x&\geqslant\log2,
\end{align}
where
\begin{align}
f|_X(\mathbf x)=
\begin{cases}
&f(\mathbf x) \quad\text{if }\mathbf x\in X, \\
&0 \quad\text{otherwise.}
\end{cases}
\end{align}
\end{theorem}

\begin{proof}
We first divide the union of $A$ and $B$ as below:
\begin{align}
A\cup B=(A\backslash B)\coprod(B\backslash A)\coprod(A\cap B),
\end{align}
where $\coprod$ represents the disjoint union.
Then one can divide the integral into three parts:
\begin{align}
-\int q_A(\mathbf x)\log\frac{q_B(\mathbf x)}{q_A(\mathbf x)+q_B(\mathbf x)}\mathrm d\mathbf x
&=-\int_{A\backslash B} q_A(\mathbf x)\log\frac{q_B(\mathbf x)}{q_A(\mathbf x)+q_B(\mathbf x)}\mathrm d\mathbf x \\
&\qquad-\int_{B\backslash A} q_A(\mathbf x)\log\frac{q_B(\mathbf x)}{q_A(\mathbf x)+q_B(\mathbf x)}\mathrm d\mathbf x \\
&\qquad-\int_{A\cap B} q_A(\mathbf x)\log\frac{q_B(\mathbf x)}{q_A(\mathbf x)+q_B(\mathbf x)}\mathrm d\mathbf x.
\end{align}

When $\mu_d(A\backslash B)=\mu_d(B\backslash A)=0$, by \Cref{prop:zero_integral}, we have
\begin{align}
-\int_{A\backslash B} q_A(\mathbf x)\log\frac{q_B(\mathbf x)}{q_A(\mathbf x)+q_B(\mathbf x)}\mathrm d\mathbf x&=0 \\
-\int_{B\backslash A} q_A(\mathbf x)\log\frac{q_B(\mathbf x)}{q_A(\mathbf x)+q_B(\mathbf x)}\mathrm d\mathbf x&=0.
\end{align}
Let $Z=\{\mathbf x\mid q_A|_{A\cap B}(\mathbf x)\neq q_B|_{A\cap B}(\mathbf x)\}$.
Note that $Z\subseteq A\cap B$, since $q_A|_{A\cap B}=q_B|_{A\cap B}\equiv0$ outside $A\cap B$.
When $\mu_d(Z)=0$, then we have
\begin{align}
&-\int_{A\cap B} q_A(\mathbf x)\log\frac{q_B(\mathbf x)}{q_A(\mathbf x)+q_B(\mathbf x)}\mathrm d\mathbf x \\
=&-\int_{(A\cap B)\backslash Z} q_A(\mathbf x)\log\frac{1}{2}\mathrm d\mathbf x-\int_Z q_A(\mathbf x)\log\frac{q_B(\mathbf x)}{q_A(\mathbf x)+q_B(\mathbf x)}\mathrm d\mathbf x-\int_Z q_A(\mathbf x)\log\frac{1}{2}\mathrm d\mathbf x\label{eq:thm2_zero_integral_1} \\
=&-\int_{A\cap B} q_A(\mathbf x)\log\frac{1}{2}\mathrm d\mathbf x+0 \\
=&-\int_{A\cap B} q_A(\mathbf x)\log\frac{1}{2}\mathrm d\mathbf x-\int_{A\backslash B} q_A(\mathbf x)\log\frac{1}{2}\mathrm d\mathbf x\label{eq:thm2_zero_integral_2} \\
=&-\int_A q_A(\mathbf x)\log\frac{1}{2}\mathrm d\mathbf x=\log2,
\end{align}
where \cref{eq:thm2_zero_integral_1,eq:thm2_zero_integral_2} are due to $\mu_d(Z)=\mu_d(A\backslash B)=0$ and \Cref{prop:zero_integral}.

On the other hand, when the inequality reaches its minimum, by the definition of the support set, we have $q_A|_{B\backslash A}\equiv0$, and $q_B|_{A\backslash B}\equiv0$.
Therefore, we have $\frac{q_B}{q_A+q_B}|_{A\backslash B}\equiv0$, and
\begin{align}
-\int_{B\backslash A} q_A(\mathbf x)\log\frac{q_B(\mathbf x)}{q_A(\mathbf x)+q_B(\mathbf x)}\mathrm d\mathbf x=0.
\end{align}
If $\mu_d(A\backslash B)>0$, then
\begin{align}
-\int_{A\backslash B} q_A(\mathbf x)\log\frac{q_B(\mathbf x)}{q_A(\mathbf x)+q_B(\mathbf x)}\mathrm d\mathbf x=+\infty,
\end{align}
which contradicts with the optimality.
Hence we prove that $\mu_d(A\backslash B)=0$.

Let $Z=\{\mathbf x\mid q_A|_{A\cap B}(\mathbf x)\neq q_B|_{A\cap B}(\mathbf x)\}$.
Then $Z\subseteq A\cap B$ since $q_A|_{A\cap B}=q_B|_{A\cap B}\equiv0$ outside $A\cap B$.
We can then deduce that
\begin{align}
&-\int q_A(\mathbf x)\log\frac{q_B(\mathbf x)}{q_A(\mathbf x)+q_B(\mathbf x)}\mathrm d\mathbf x \\
=&-\int_{A\cap B} q_A(\mathbf x)\log\frac{q_B(\mathbf x)}{q_A(\mathbf x)+q_B(\mathbf x)}\mathrm d\mathbf x-\int_{A\backslash B} q_A(\mathbf x)\log\frac{q_B(\mathbf x)}{q_A(\mathbf x)+q_B(\mathbf x)}\mathrm d\mathbf x \\
=&-\int_{A\cap B} q_A(\mathbf x)\log\frac{q_B(\mathbf x)}{q_A(\mathbf x)+q_B(\mathbf x)}\mathrm d\mathbf x\label{eq:thm2_zero_integral_3} \\
=&-\int_{(A\cap B)\backslash Z} q_A(\mathbf x)\log\frac{q_B(\mathbf x)}{q_A(\mathbf x)+q_B(\mathbf x)}\mathrm d\mathbf x-\int_Z q_A(\mathbf x)\log\frac{q_B(\mathbf x)}{q_A(\mathbf x)+q_B(\mathbf x)}\mathrm d\mathbf x \\
=&\log2\left(\int_{(A\cap B)\backslash Z} q_A(\mathbf x)\mathrm d\mathbf x\right)-\int_Z q_A(\mathbf x)\log\frac{q_B(\mathbf x)}{q_A(\mathbf x)+q_B(\mathbf x)}\mathrm d\mathbf x \\
=&\log2\left(\int_{(A\cap B)\backslash Z} q_A(\mathbf x)\mathrm d\mathbf x+\int_{(A\backslash B)\backslash Z} q_A(\mathbf x)\mathrm d\mathbf x\right)-\int_Z q_A(\mathbf x)\log\frac{q_B(\mathbf x)}{q_A(\mathbf x)+q_B(\mathbf x)}\mathrm d\mathbf x\label{eq:thm2_zero_integral_4} \\
=&\log2\left(\int_{A\backslash Z} q_A(\mathbf x)\mathrm d\mathbf x\right)-\int_Z q_A(\mathbf x)\log\frac{q_B(\mathbf x)}{q_A(\mathbf x)+q_B(\mathbf x)}\mathrm d\mathbf x,
\end{align}
where \cref{eq:thm2_zero_integral_3,eq:thm2_zero_integral_4} are due to $\mu_d((A\backslash B)\backslash Z)\leqslant\mu_d(A\backslash B)=0$ and \Cref{prop:zero_integral}.
Suppose $\mu_d(Z)>0$, then $C_A=\int_Zq_A(\mathbf x)\mathrm d\mathbf x>0,C_B=\int_Zq_B(\mathbf x)\mathrm d\mathbf x>0$.
By the definition of $Z$ and $\mu_d(A\backslash B)=0$, we have
\begin{align}
C_A&=\int_Zq_A(\mathbf x)\mathrm d\mathbf x+\left(\int_{(A\cap B)\backslash Z}q_A(\mathbf x)\mathrm d\mathbf x-\int_{(A\cap B)\backslash Z}q_B(\mathbf x)\mathrm d\mathbf x\right)+\int_{A\backslash B}q_A(\mathbf x)\mathrm d\mathbf x-\int_{A\backslash B}q_B(\mathbf x)\mathrm d\mathbf x\label{eq:thm2_zero_integral_5} \\
&=\int_{A\cap B}q_A(\mathbf x)\mathrm d\mathbf x-\int_{(A\cap B)\backslash Z}q_B(\mathbf x)\mathrm d\mathbf x+\int_{A\backslash B}q_A(\mathbf x)\mathrm d\mathbf x-\int_{A\backslash B}q_B(\mathbf x)\mathrm d\mathbf x \\
&=\int_Aq_A(\mathbf x)\mathrm d\mathbf x-\int_{(A\cap B)\backslash Z}q_B(\mathbf x)\mathrm d\mathbf x-\int_{A\backslash B}q_B(\mathbf x)\mathrm d\mathbf x \\
&\geqslant\int_Aq_B(\mathbf x)\mathrm d\mathbf x-\int_{(A\cap B)\backslash Z}q_B(\mathbf x)\mathrm d\mathbf x-\int_{A\backslash B}q_B(\mathbf x)\mathrm d\mathbf x \\
&=\int_Zq_B(\mathbf x)\mathrm d\mathbf x=C_B,
\end{align}
where \cref{eq:thm2_zero_integral_5} is due to $q_A=q_B$ on $(A\cap B)\backslash Z$, and $\mu_d(A\backslash B)=0$ with \Cref{prop:zero_integral}.

Then we have
\begin{align}
&-\int_Z q_A(\mathbf x)\log\frac{q_B(\mathbf x)}{q_A(\mathbf x)+q_B(\mathbf x)}\mathrm d\mathbf x \\
=&-C_A\int_Z\frac{1}{C_A}q_A(\mathbf x)\log\frac{\frac{1}{C_A}q_B(\mathbf x)}{\frac{1}{C_A}q_A(\mathbf x)+\frac{1}{C_A}q_B(\mathbf x)}\mathrm d\mathbf x \\
\geqslant&-C_A\int_Z\frac{1}{C_A}q_A(\mathbf x)\log\frac{\frac{1}{C_B}q_B(\mathbf x)}{\frac{1}{C_A}q_A(\mathbf x)+\frac{1}{C_B}q_B(\mathbf x)}\mathrm d\mathbf x.
\end{align}
Note that $\int_Z\frac{1}{C_A}q_A(\mathbf x)\mathrm d\mathbf x=\int_Z\frac{1}{C_B}q_B(\mathbf x)\mathrm d\mathbf x=1$, one can rewrite $q_A'=\frac{1}{C_A}q_A,q_B'=\frac{1}{C_B}q_B,$ then we have
\begin{align}
&-C_A\int_Z\frac{1}{C_A}q_A(\mathbf x)\log\frac{\frac{1}{C_B}q_B(\mathbf x)}{\frac{1}{C_A}q_A(\mathbf x)+\frac{1}{C_B}q_B(\mathbf x)}\mathrm d\mathbf x \\
=&-C_A\int_Zq_A'(\mathbf x)\log\frac{q_B'(\mathbf x)}{q_A'(\mathbf x)+q_B'(\mathbf x)}\mathrm d\mathbf x>C_A\log2 \label{eq:gan_ineq},
\end{align}
where \cref{eq:gan_ineq} is due to the property of generator loss on two distinct nonzero distributions.
Therefore, we have the contradiction:
\begin{align}
&-\int q_A(\mathbf x)\log\frac{q_B(\mathbf x)}{q_A(\mathbf x)+q_B(\mathbf x)}\mathrm d\mathbf x \\
=&\log2\left(\int_{A\backslash Z} q_A(\mathbf x)\mathrm d\mathbf x\right)-\int_Z q_A(\mathbf x)\log\frac{q_B(\mathbf x)}{q_A(\mathbf x)+q_B(\mathbf x)}\mathrm d\mathbf x \\
>&\log2\left(\int_{A\backslash Z} q_A(\mathbf x)\mathrm d\mathbf x\right)+C_A\log2 \\
=&\log2\left(\int_{A\backslash Z} q_A(\mathbf x)\mathrm d\mathbf x\right)+\log2\left(\int_Zq_A(\mathbf x)\mathrm d\mathbf x\right) \\
=&\log2\left(\int_A q_A(\mathbf x)\mathrm d\mathbf x\right)=\log2,
\end{align}
which indicates that $\mu_d(Z)=0$.

Finally, it suffices to show $\mu_d(B\backslash A)=0$.
If $\mu_d(B\backslash A)>0$, then $\int_{B\backslash A}q_B(\mathbf x)\mathrm d\mathbf x>0$
\begin{align}
C_B&=\int_Aq_B(\mathbf x)\mathrm d\mathbf x-\int_{(A\cap B)\backslash Z}q_B(\mathbf x)\mathrm d\mathbf x-\int_{A\backslash B}q_B(\mathbf x)\mathrm d\mathbf x \\
&=\left(\int_{A\backslash B}q_B(\mathbf x)\mathrm d\mathbf x+\int_{A\cap B}q_B(\mathbf x)\mathrm d\mathbf x\right)-\int_{(A\cap B)\backslash Z}q_B(\mathbf x)\mathrm d\mathbf x-\int_{A\backslash B}q_B(\mathbf x)\mathrm d\mathbf x \\
&=\int_{A\cap B}q_B(\mathbf x)\mathrm d\mathbf x-\int_{(A\cap B)\backslash Z}q_B(\mathbf x)\mathrm d\mathbf x-\int_{A\backslash B}q_B(\mathbf x)\mathrm d\mathbf x\label{eq:thm2_zero_integral_6} \\
&=\left(\int_Bq_B(\mathbf x)\mathrm d\mathbf x-\int_{B\backslash A}q_B(\mathbf x)\mathrm d\mathbf x\right)-\int_{(A\cap B)\backslash Z}q_B(\mathbf x)\mathrm d\mathbf x-\int_{A\backslash B}q_B(\mathbf x)\mathrm d\mathbf x \\
&<\int_Bq_B(\mathbf x)\mathrm d\mathbf x-\int_{(A\cap B)\backslash Z}q_B(\mathbf x)\mathrm d\mathbf x-\int_{A\backslash B}q_B(\mathbf x)\mathrm d\mathbf x \\
&=\int_Aq_A(\mathbf x)\mathrm d\mathbf x-\int_{(A\cap B)\backslash Z}q_B(\mathbf x)\mathrm d\mathbf x-\int_{A\backslash B}q_B(\mathbf x)\mathrm d\mathbf x \\
&=\int_Aq_A(\mathbf x)\mathrm d\mathbf x-\int_{(A\cap B)\backslash Z}q_A(\mathbf x)\mathrm d\mathbf x=C_A\label{eq:thm2_zero_integral_7},
\end{align}
where \cref{eq:thm2_zero_integral_6,eq:thm2_zero_integral_7} is due to $\mu_d(A\backslash B)=0$ with \Cref{prop:zero_integral} and $q_A=q_B$ on $(A\cap B)\backslash Z$.
Then we have
\begin{align}
&-\int_Z q_A(\mathbf x)\log\frac{q_B(\mathbf x)}{q_A(\mathbf x)+q_B(\mathbf x)}\mathrm d\mathbf x \\
>&-C_A\int_Z\frac{1}{C_A}q_A(\mathbf x)\log\frac{\frac{1}{C_B}q_B(\mathbf x)}{\frac{1}{C_A}q_A(\mathbf x)+\frac{1}{C_B}q_B(\mathbf x)}\mathrm d\mathbf x \\
=&-C_A\int_Zq_A'(\mathbf x)\log\frac{q_B'(\mathbf x)}{q_A'(\mathbf x)+q_B'(\mathbf x)}\mathrm d\mathbf x=C_A\log2,
\end{align}
and $-\int q_A(\mathbf x)\log\frac{q_B(\mathbf x)}{q_A(\mathbf x)+q_B(\mathbf x)}\mathrm d\mathbf x>\log2\left(\int_{A\backslash Z} q_A(\mathbf x)\mathrm d\mathbf x\right)+C_A\log2=\log2$.
Therefore $\mu_d(B\backslash A)=0$.
\end{proof}

\subsection{Proof of \Cref{thm:score_matching}}\label{subsec:a.3}

\begin{theorem}
Denote by $\mathrm{dist}(\mathbf x)$ the distance between $\mathbf x$ and $\mathrm{supp}\;q_0$.
For any $\mathbf y\notin\mathrm{supp}\;q_0$, define a sequence of random variable $\mathbf y_0=\mathbf y$, $\mathbf y_{k+1}=\mathcal R(\mathbf y_k,\boldsymbol\epsilon_k,t)$ with $\boldsymbol\epsilon_k\sim\mathcal N(\mathbf 0,\mathbf I)$.
Then the sequence $\{\mathbf y_k\}_{k=0}^{\infty}$ converges to $\mathrm{supp}\;q_0$, \textit{i.e.},
\begin{align}
\lim_{k\rightarrow+\infty,t\rightarrow0}\mathrm{dist}(\mathbf y_k)=0.
\end{align}
\end{theorem}

\begin{proof}
Note that $\boldsymbol\epsilon_\theta(\alpha_t\mathbf x+\sigma_t\boldsymbol\epsilon,t)=-\sigma_t\nabla\log q_t(\mathbf x_t)$, therefore we can rewrite the one-step refinement $\mathcal R$ as below:
\begin{align}
\mathcal R(\mathbf x,\boldsymbol\epsilon,t)&=\mathbf x+\frac{\sigma_t}{\alpha_t}(\boldsymbol\epsilon-\boldsymbol\epsilon_\theta(\alpha_t\mathbf x+\sigma_t\boldsymbol\epsilon,t)) \\
&=\mathbf x+\frac{\sigma_t^2}{\alpha_t}\nabla\log q_t(\alpha_t\mathbf x+\sigma_t\boldsymbol\epsilon)+\frac{\sigma_t}{\alpha_t}\boldsymbol\epsilon
\end{align}

Recall the corresponding SDE of the reverse process of DPMs
\begin{align}\label{eq:reverse_sde}
\mathrm d\mathbf x_t=f(t)\mathbf x_t\mathrm dt-g^2(t)\nabla_{\mathbf x_t}\log q_t(\mathbf x_t)\mathrm dt+g(t)\mathrm d\mathbf w.
\end{align}
One can refer to $\mathcal R$ as a discretization of \cref{eq:reverse_sde}.
Then the conclusion comes directly as a deduction of the solution to this SDE in \cref{eq:reverse_sde}, since the limit for $k\rightarrow+\infty,t\rightarrow 0$ indicates the continuous version of this SDE and cancels the discretization error.
\end{proof}

\begin{remark}
As \Cref{thm:score_matching} concludes, the sequence of refined results will converge to locate at the support of the data distribution $\mathrm{supp}\;q_0$.
By the Cauchy's convergence law, we claim that for arbitrarily small $\varepsilon>0$, there exists $K>0$ such that for any $k>K$, we have $\|\mathbf y_{k+1}-\mathbf y_k\|_2<\varepsilon$.
This indicates that there will be no gradient when synthesized samples support on the data manifold.
Otherwise, the gradient of nonzero $\|\mathbf y_{k+1}-\mathbf y_k\|_2$ will enforce the convergence of the refinement sequence towards the data manifold.
\end{remark}

\subsection{Proof of \Cref{thm:cdscore_matching}}\label{subsec:a.4}

Before addressing the feasibility theorem under conditional generation setting, we first define the conditional one-step refinement as below
\begin{align}\label{eq:cdrefine}
\mathcal R(\mathbf x,\boldsymbol\epsilon,c,t):=\mathbf x+\frac{\sigma_t}{\alpha_t}(\boldsymbol\epsilon-\boldsymbol\epsilon_\theta(\alpha_t\mathbf x+\sigma_t\boldsymbol\epsilon,c,t)).
\end{align}

\begin{theorem}\label{thm:cdscore_matching}
Denote by $\mathrm{dist}(\mathbf x,c)$ the distance between pair $(\mathbf x,c)$ and $\mathrm{supp}\;q_0$.
For any pair $(\mathbf y,c)\notin\mathrm{supp}\;q_0$, define a sequence of random variable $\mathbf y_0=\mathbf y$, $\mathbf y_{k+1}=\mathcal R(\mathbf y_k,\boldsymbol\epsilon_k,c,t)$ with $\boldsymbol\epsilon_k\sim\mathcal N(\mathbf 0,\mathbf I)$.
Then the sequence $\{(\mathbf y_k,c)\}_{k=0}^{\infty}$ converges to $\mathrm{supp}\;q_0$, \textit{i.e.},
\begin{align}
\lim_{k\rightarrow+\infty,t\rightarrow0}\mathrm{dist}(\mathbf y_k,c)=0.
\end{align}
\end{theorem}

\begin{proof}
Denote by $q_0^c(\mathbf x)=q_0(\mathbf x, c)/q(c)=q_0(\mathbf x|c)$, and by $\boldsymbol\epsilon_\theta^c(\mathbf x_t,t)=\boldsymbol\epsilon_\theta(\mathbf x_t,c,t)$ for any condition $c$, one can refer to $\boldsymbol\epsilon_\theta^c$ as the ground-truth noise prediction model pre-trained on the data distribution $q_0^c(\mathbf x)$.
Denote by $\mathcal R^c(\mathbf x,\boldsymbol\epsilon,t)=\mathcal R(\mathbf x,\boldsymbol\epsilon,c,t)$ the refinement involving $\boldsymbol\epsilon_\theta^c$, and by $\mathrm{dist}^c(\mathbf x)=\mathrm{dist}(\mathbf x,c)$.
Then by \Cref{thm:score_matching}, one can conclude that for any $\mathbf y\notin\mathrm{supp}\;q_0^c$, and a sequence of random variable $\mathbf y_0=\mathbf y$, $\mathbf y_{k+1}=\mathcal R^c(\mathbf y_k,\boldsymbol\epsilon_k,t)$ with $\boldsymbol\epsilon_k\sim\mathcal N(\mathbf 0,\mathbf I)$, we have
\begin{align}
\lim_{k\rightarrow+\infty,t\rightarrow0}\mathrm{dist}(\mathbf y_k,c)=\lim_{k\rightarrow+\infty,t\rightarrow0}\mathrm{dist}^c(\mathbf y_k)=0.
\end{align}
And by the definition of $q_0^c$, $\mathbf y\notin\mathrm{supp}\;q_0^c$ implies that $(\mathbf y,c)\notin\mathrm{supp}\;q_0$, and $\mathbf x\in\mathrm{supp}\;q_0^c$ implies that $(\mathbf x,c)\in\mathrm{supp}\;q_0$.
\end{proof}

\section{Detailed implementation of \method}

\subsection{Empirical Value of Hyper-parameters of \method}\label{subsec:b.1}

\begin{table*}[!ht]
\caption{
    \textbf{Empirical value} of hyper-parameters for \method used in our experiments.
}
\label{tab:hyperparameters}
\vskip 0.15in
\centering
\footnotesize
\SetTblrInner{rowsep=1.0pt}      
\SetTblrInner{colsep=20.0pt}     
\begin{tblr}{
    cell{1-6}{2-5}={halign=c,valign=m},         
    cell{1-6}{1}={halign=c,valign=m},           
    hline{1,7}={1-5}{1.0pt},                    
    hline{2}={1-5}{},                           
}
Dataset           &      CIFAR10 &   ImageNet 64 &  ImageNet 128 &  LSUN Bedroom \\
Setting           &  Conditional &   Conditional &   Conditional & Unconditional \\
Dataset Scale     & $50K$ Images & $1.3M$ Images & $1.3M$ Images &   $3M$ Images \\
$\lambda_{score}$ &         0.01 &           0.1 &           0.1 &           0.1 \\
$t$               &    $[40,60]$ &    $[25,35]$  &    $[25,35]$  &     $[25,35]$ \\
Frequency         &            8 &             8 &             8 &             8 \\
\end{tblr}
\vspace{-10pt}
\end{table*}

First, we would like to summarize the guideline if choosing the adequate hyper-parameters as below:

\begin{itemize}
\item It is recommended to choose $\lambda_{score}=0.01$ for small-scale dataset (\textit{e.g.}, CIFAR10) and $\lambda_{score}=0.1$ for large-scale dataset (\textit{e.g.}, ImageNet);
\item Narrowed timestep interval is suggested to set to near 50 for commonly used diffusion models with total timesteps $T=1000$;
\item Regularization frequency is suggested to set as 8.
\end{itemize}

Second, empirical value of hyper-parameters used in our experiments are listed in \cref{tab:hyperparameters}.
We hope these values can help users to efficiently find a combination for a new dataset.

\subsection{Training Cost on Aurora}

\begin{table*}[t]
\caption{
    \textbf{Comparison} of computational cost on Aurora~\cite{zhu2023aurora} on ImageNet 64x64~\cite{dengjia2009}.
    We involve the trick to omit U-Net Jacobian in DreamFusion~\cite{poole2023dreamfusion} for better training efficiency.
    We report average training time for one iteration, maximal GPU memory, and number of GPUs, respectively.
}
\label{tab:time_cost}
\vskip 0.15in
\centering
\footnotesize
\SetTblrInner{rowsep=1.0pt}      
\SetTblrInner{colsep=8.0pt}      
\begin{tblr}{
    cell{1-5}{2-5}={halign=c,valign=m},         
    cell{1-5}{1}={halign=l,valign=m},           
    hline{1,6}={1-5}{1.0pt},                    
    hline{2}={1-5}{},                           
}
Method                                     & Batch Size & Average Iteration Time & Max GPU Memory & \# GPUs \\
Aurora~\cite{zhu2023aurora}                &       1024 &                   3.3s &       36.10 GB &      16 \\
Aurora~\cite{zhu2023aurora}                &       1024 &                   1.8s &       36.10 GB &      32 \\
Aurora + \method                           &       1024 &                   2.9s &       56.00 GB &      32 \\
Aurora + \method (omitting U-Net Jacobian) &       1024 &                   2.0s &       37.22 GB &      32 \\
\end{tblr}
\vspace{-10pt}
\end{table*}

We further report in \cref{tab:time_cost} the average training time for one iteration, and number of used GPUs, respectively.
When keeping the same batch size and the same number of GPUs, our \method sightly increases the training cost.
Note that we double the number of GPUs to avoid memory explosion at the training steps when we apply the proposed \method.
The batch size is kept the same as the baseline.

\section{Additional Samples from \method}

In this section, we provide additional samples from \method, including diverse synthesis (\textit{i.e.}, \cref{fig:bedroom_supp,fig:imagenet_biggan_supp,fig:imagenet_supp}) and latent interpolation (\textit{i.e.}, \cref{fig:interp_bedroom_supp,fig:interp_imagenet_supp}).
All samples are synthesized with \method upon StyleGAN2~\cite{Karras2019AnalyzingAI} on LSUN Bedroom 256x256~\cite{yu15lsun}, BigGAN~\cite{Brock2018LargeSG} on ImageNet 128x128~\cite{dengjia2009} and Aurora~\cite{zhu2023aurora} on ImageNet 64x64~\cite{dengjia2009}, respectively.

\begin{figure*}[t]
\centering
\includegraphics[width=1.0\textwidth]{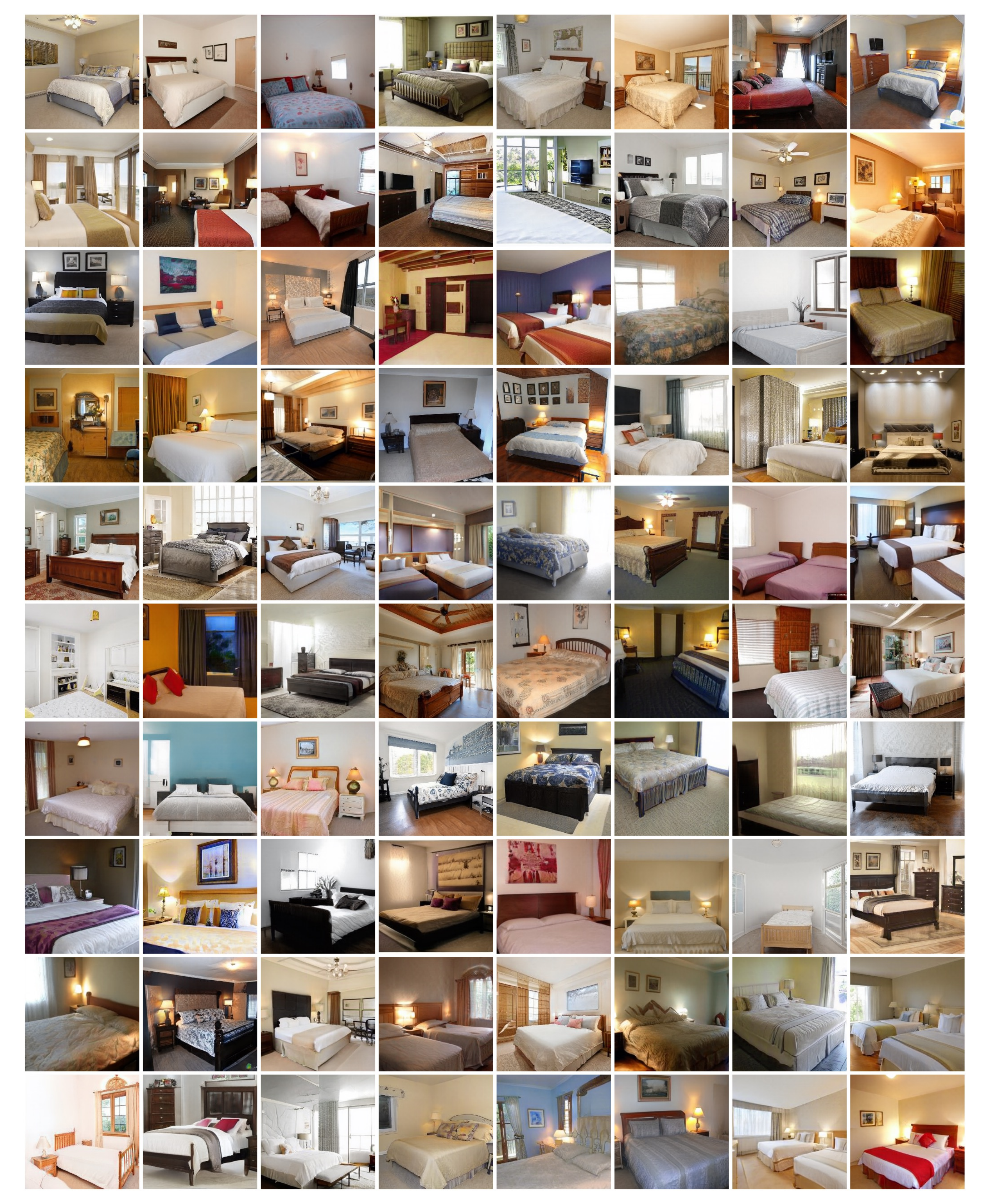}
\vspace{-18pt}
\caption{
    \textbf{Diverse results} generated by \method upon StyleGAN2~\cite{Karras2019AnalyzingAI} trained on LSUN Bedroom 256x256 dataset~\cite{yu15lsun}.
    We randomly sample the global latent code $\mathbf z$ for each image.
}
\label{fig:bedroom_supp}
\end{figure*}

\begin{figure*}[t]
\centering
\includegraphics[width=1.0\textwidth]{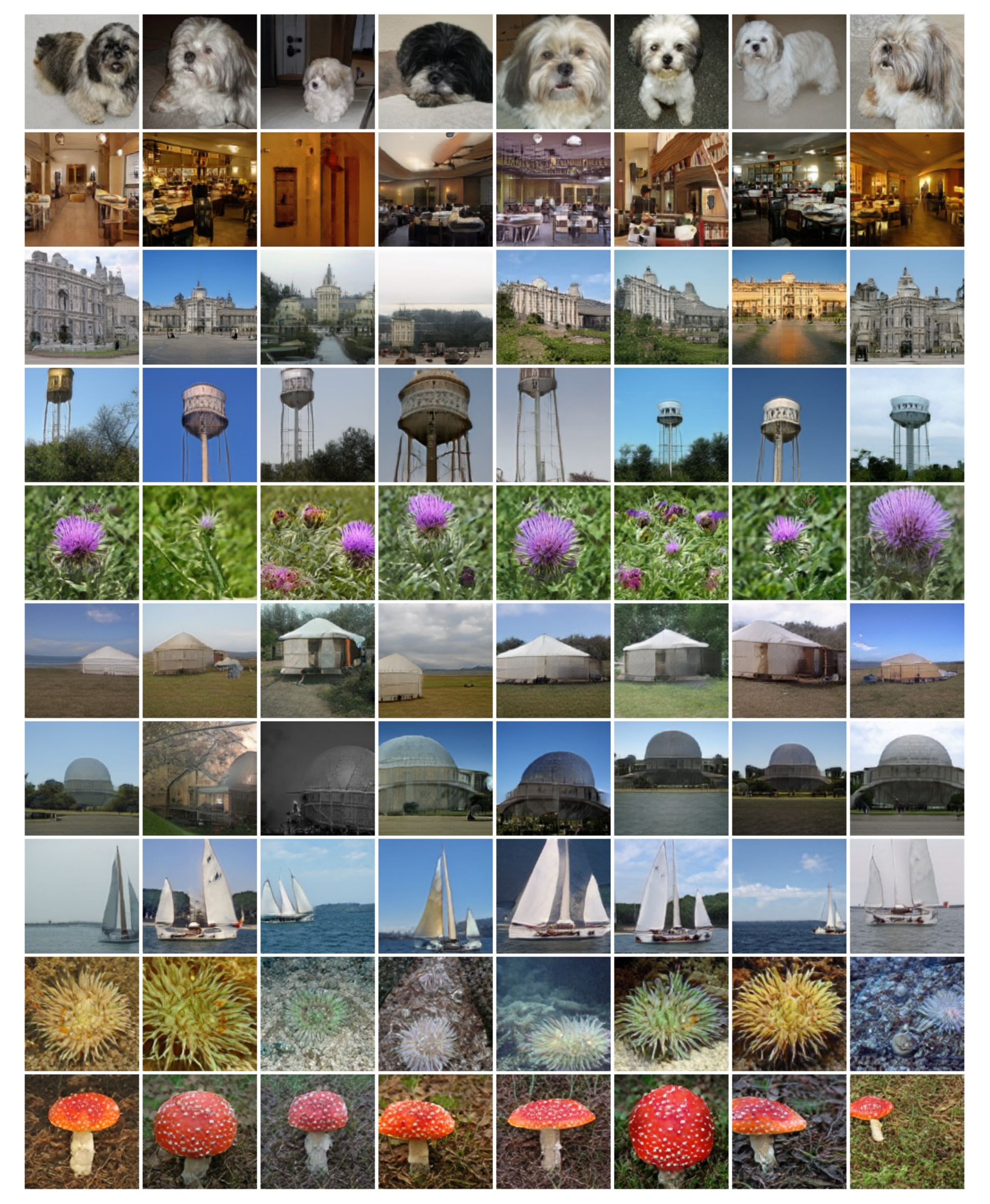}
\vspace{-18pt}
\caption{
    \textbf{Diverse results} generated by \method upon BigGAN~\cite{Brock2018LargeSG} trained on ImageNet 128x128 dataset~\cite{dengjia2009}
    We randomly sample eight global latent codes $\mathbf z$ for each label condition $c$, demonstrated in each row.
}
\label{fig:imagenet_biggan_supp}
\end{figure*}

\begin{figure*}[t]
\centering
\includegraphics[width=1.0\textwidth]{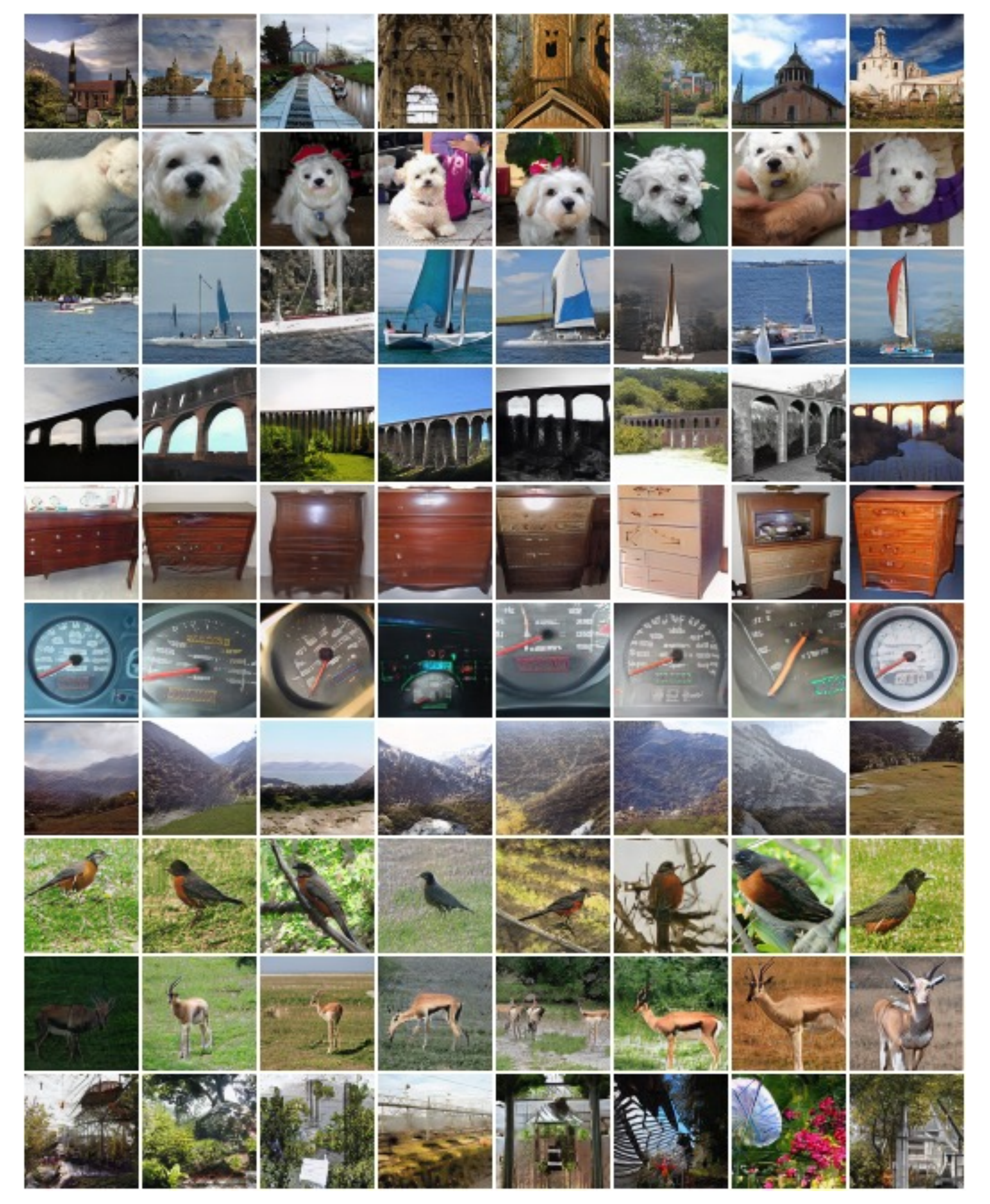}
\vspace{-18pt}
\caption{
    \textbf{Diverse results} generated by \method upon Aurora~\cite{zhu2023aurora} trained on ImageNet 64x64 dataset~\cite{dengjia2009}
    We randomly sample eight global latent codes $\mathbf z$ for each label condition $c$, demonstrated in each row.
}
\label{fig:imagenet_supp}
\end{figure*}

\begin{figure*}[t]
\centering
\includegraphics[width=1.0\textwidth]{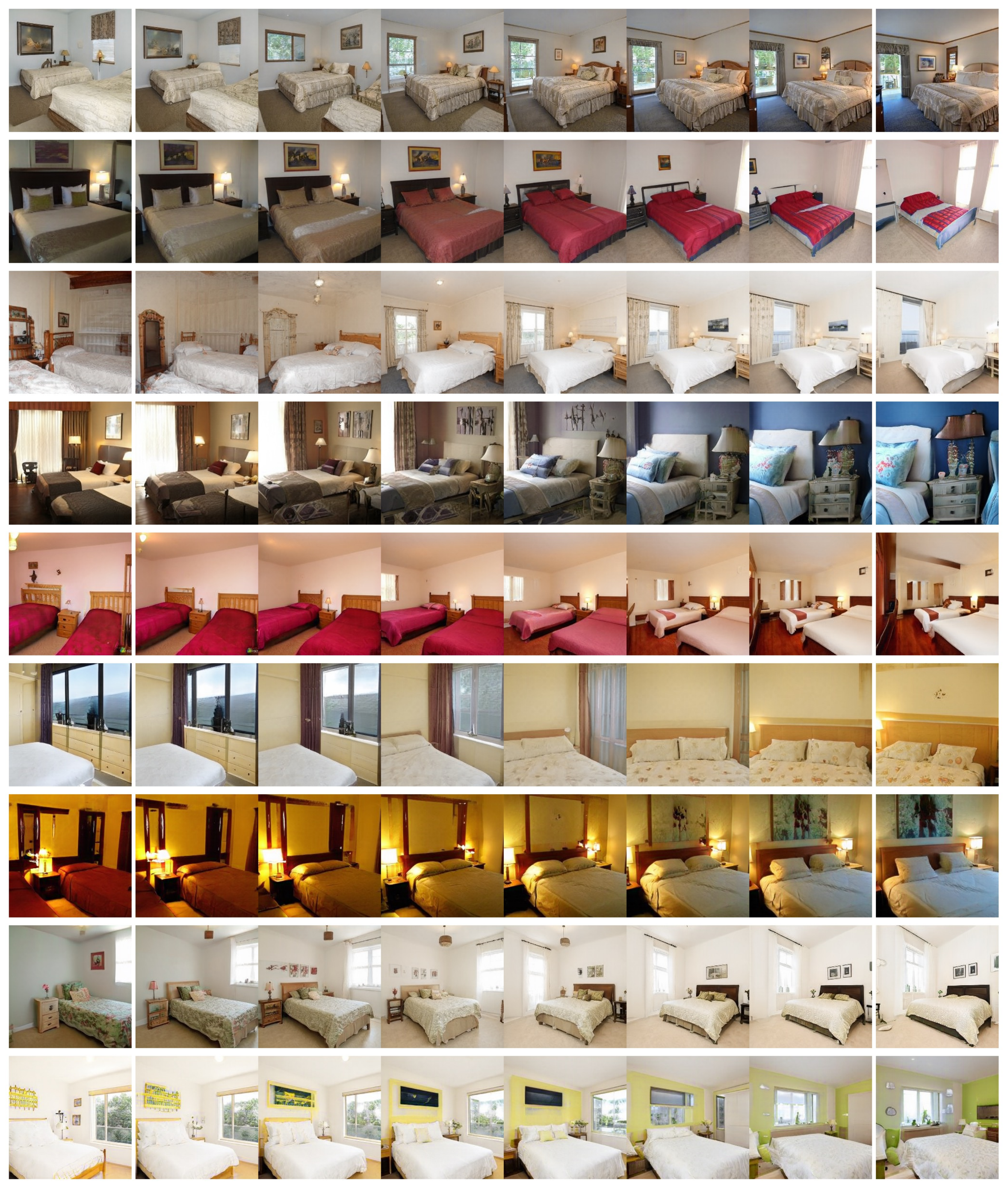}
\vspace{-18pt}
\caption{
    \textbf{Interpolation} between leftmost and rightmost images with linear interpolation.
    We apply \method upon StyleGAN2~\cite{Karras2019AnalyzingAI} on LSUN Bedroom 256x256 dataset~\cite{yu15lsun}, interpolating in the disentangled latent space.
}
\label{fig:interp_bedroom_supp}
\end{figure*}

\begin{figure*}[t]
\centering
\includegraphics[width=1.0\textwidth]{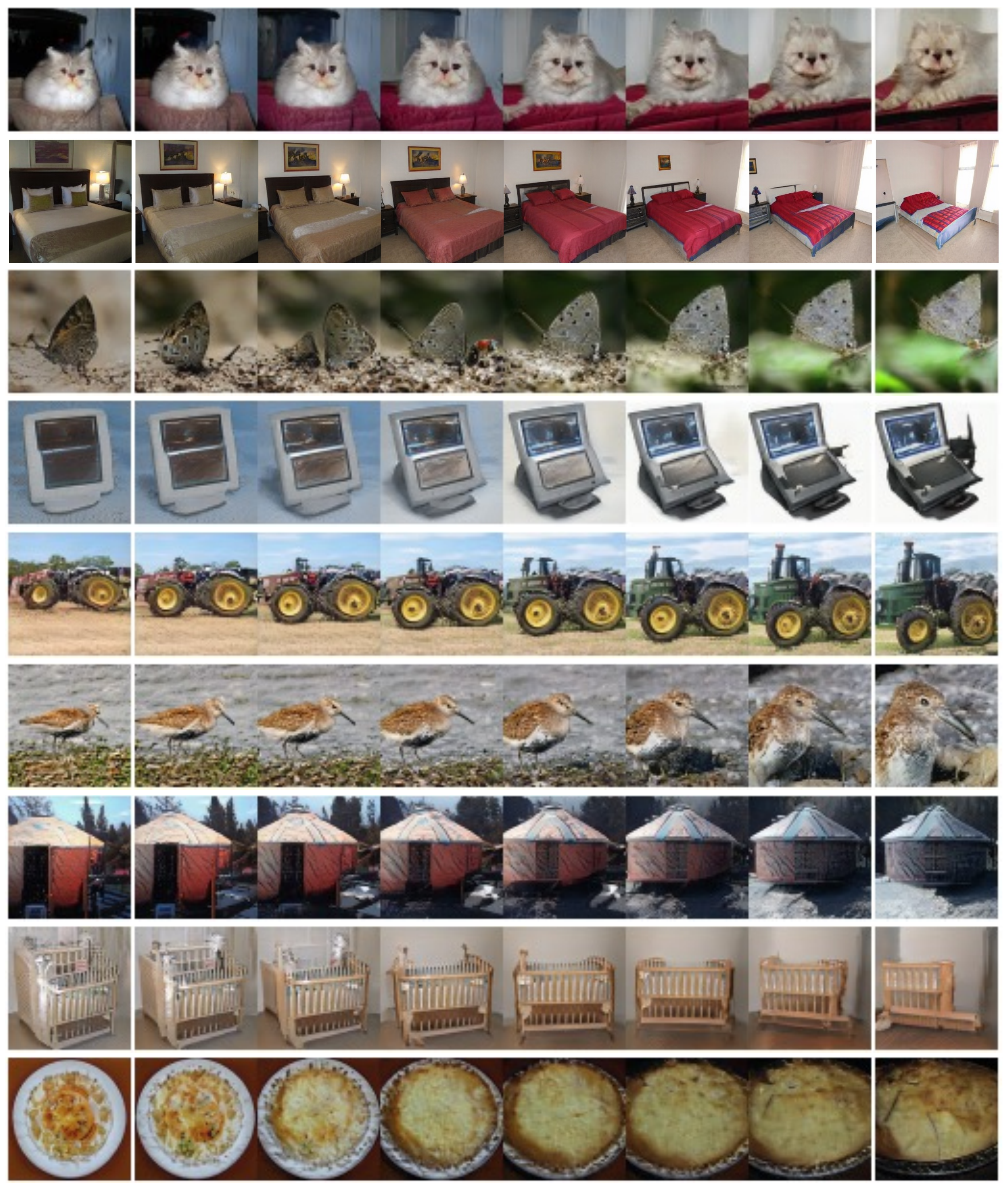}
\vspace{-18pt}
\caption{
    \textbf{Interpolation} between leftmost and rightmost images with linear interpolation.
    We apply \method upon Aurora~\cite{zhu2023aurora} on ImageNet 64x64 dataset~\cite{dengjia2009}, interpolating in the disentangled latent space by fixing the label condition.
}
\label{fig:interp_imagenet_supp}
\end{figure*}


\end{document}